\documentclass{article}
\usepackage{ijcai26}

\usepackage{times}
\usepackage{soul}
\usepackage{url}
\usepackage[hidelinks]{hyperref}
\usepackage[utf8]{inputenc}
\usepackage[small]{caption}
\usepackage{graphicx}
\usepackage{amsmath}
\usepackage{amsthm}
\usepackage{booktabs}
\usepackage{algorithm}
\usepackage{algorithmic}
\usepackage[switch]{lineno}
\usepackage{multirow}
\usepackage{pifont}
\usepackage{amssymb}

\newtheorem{theorem}{Theorem}
\newtheorem{remark}{Remark}

\title{LEGATO: Good Identity Unlearning Is Continuous}

\author{
Qiang Chen$^{1,3}$\thanks{Equal contribution}
\and
Chun-Wun Cheng$^{2}$\footnotemark[1]
\and
Xiu Su$^3$\thanks{Corresponding author}
\and
Hongyan Xu$^3$
\and
\\
Xi Lin$^4$
\and
Shan You$^5$ 
\and
Angelica I. Aviles-Rivero$^6$
\and
Yi Chen$^1$\footnotemark[2]
\\
\affiliations
$^1$HKUST\\
$^2$University of Cambridge\\
$^3$Central South University\\
$^4$Shanghai Jiaotong University\\
$^5$SenseTime Research\\
$^6$Tsinghua University\\
\emails
qiangchen.sh@gmail.com,
cwc56@cam.ac.uk,
\{xiusu1994, hongyanxu\}@csu.edu.cn,
linxi234@sjtu.edu.cn,
youshan@senseauto.com,
aviles-rivero@tsinghua.edu.cn,
yichen@ust.hk
}

\begin{document}

\maketitle

\begin{abstract}
Machine unlearning has become a crucial role in enabling generative models trained on large datasets to remove sensitive, private, or copyright-protected data. However, existing machine unlearning methods face three challenges in learning to forget identity of generative models: 1) inefficient, where identity erasure requires fine-tuning all the model's parameters; 2) limited controllability, where forgetting intensity cannot be controlled and explainability is lacking; 3) catastrophic collapse, where the model's retention capability undergoes drastic degradation as forgetting progresses. Forgetting has typically been handled through discrete and unstable updates, often requiring full-model fine-tuning and leading to catastrophic collapse. \textbf{In this work, we argue that identity forgetting should be modeled as a continuous trajectory}, and introduce LEGATO —  \textbf{L}earn to Forg\textbf{E}t Identity in \textbf{G}ener\textbf{A}tive Models via \textbf{T}rajectory-consistent Neural \textbf{O}rdinary Differential Equations. LEGATO augments pre-trained generators with fine-tunable lightweight Neural ODE adapters, enabling smooth, controllable forgetting while keeping the original model weights frozen. This formulation allows forgetting intensity to be precisely modulated via ODE step size, offering interpretability and robustness. To further ensure stability, we introduce trajectory consistency constraints that explicitly prevent catastrophic collapse during unlearning. Extensive experiments across in-domain and out-of-domain identity unlearning benchmarks show that LEGATO achieves state-of-the-art forgetting performance, avoids catastrophic collapse and reduces fine-tuned parameters. Codes are available at https://github.com/sh-qiangchen/LEGATO.
\end{abstract}

\section{Introduction}
Recently, deep generative models \cite{rezende2014stochastic,goodfellow2014generative,karras2019style,karras2020analyzing,ho2020denoising,song2021score,rombach2022high} pre-trained on massive datasets have attracted widespread attention due to their excellent generation capabilities. However, this capability raises significant concerns, as training corpora contain sensitive, private, or copyright-protected information, potentially leading to privacy-related issues \cite{lukas2023analyzing,carlini2023extracting}. For instance, Deepfakes \cite{xu2023tall,yan2023ucf} can generate inappropriate content involving real individuals (e.g., nude celebrities). Faced with growing concerns over data privacy, regulations such as GDPR \cite{mantelero2013eu} and CCPA \cite{ccpa2018} require applications to support the removal of privacy-related content from training data, strengthening the Right to be Forgotten. Therefore, to protect a specific identity's privacy, a generative model must intentionally suppress or unlearn its distinctive features. This has motivated a line of research on machine unlearning \cite{nguyen2022survey,shaik2024exploring} of generative models. Moreover, generative unlearning is also highly valuable for removing inaccurate or outdated information contained in training data.

Exact machine unlearning involves retraining the model from scratch after removing the undesirable data, thereby guaranteeing the complete elimination of its influence. However, retraining is computationally intensive \cite{brophy2021machine,sekhari2021remember}, identifying and isolating specific subsets from large-scale datasets can also be prohibitively time-consuming. Recently, several approximate machine unlearning methods \cite{fan2024salun,li2024machine,wu2025unlearning,chen2025score,Shaheryar2025Unlearn} propose to forget specific data for generative models through directly fine-tuning the pre-trained model. Specially, \cite{li2024machine} proposed achieving unlearning in text-to-image generative models by aligning the embeddings of forgotten samples with Gaussian noise, while preserving the embedding consistency between the target and original model on the retain set. GUIDE \cite{seo2024generative} was the first to propose generative identity unlearning, which focus on removing the whole identity associated with a given single image from the generator while preserving the generative capability of the pre-trained model for other identities. Compared to machine unlearning in image-to-image \cite{krishnan2019boundless,chang2022maskgit} or text-to-image \cite{rombach2022high,singh2024negative,yang2025learn} generative models, generative identity unlearning remains largely unexplored.

While promising, these approximate machine unlearning methods in generative models still exhibit three issues. First, fine-tuning all the model's parameters still involves a large computational cost, which increases as the model size grows, and updating too many parameters can easily compromise the learned generative capability of the model. Second, the controllability and explainability of the model are limited, as the intensity of forgetting throughout the unlearning process cannot be effectively controlled. Third, forgetting stability is uncontrollable, easily leading to catastrophic collapse, where the model’s retention capability undergoes drastic degradation as forgetting progresses. GUIDE \cite{seo2024generative}, which introduced the task of generative identity unlearning, reflects many of these limitations. It requires full-model fine-tuning, offers no control over forgetting intensity, suffers from catastrophic collapse, and lacks safeguards against instability during unlearning.

To address the above challenges, we introduce LEGATO (Learn to forgEt identity in GenerAtive models via Trajectory-consistent neural Ordinary differential equations), a framework that formulates identity unlearning as a continuous transformation in the generator’s latent space. Rather than fine-tuning the full model, LEGATO adds fine-tunable lightweight Neural ODE adapters after each resolution stage, allowing targeted identity forgetting while keeping the original weights frozen. Neural Ordinary Differential Equations (Neural ODE) \cite{chen2018neural} recast a neural network as a continuous‑time dynamical system: the network’s “layers” become the hidden state of an ordinary differential equation. It provides a theoretical understanding that is more robust and invertible.
This design enables explicit control over forgetting intensity via the ODE step size, improves interpretability, and significantly reduces the number of trainable parameters. To further stabilize the process, we introduce a trajectory consistency constraint that regularizes the ODE dynamics and helps prevent catastrophic collapse. \textit{LEGATO is, to our knowledge, the first to apply Neural ODEs to machine unlearning and to treat identity forgetting as a continuous-time process.} Our contributions are as follows:

\begin{itemize}
\item We introduce a novel formulation of identity unlearning as a continuous transformation in latent space, implemented via lightweight Neural ODE adapters inserted into a pre-trained generator. This enables modular, parameter-efficient forgetting without updating the original model weights.
\item Our method allows explicit control over forgetting intensity by adjusting the ODE integration step size, providing fine-grained controllability and interpretability throughout the unlearning process. 
\item Theoretically, we proved the smoothness of neural ODE trajectories, non-monotonicity of step size and existence of an optimal interval in identity unlearning, and the feasibility of conflict-free multi-identity unlearning.
\item We propose enforcing trajectory consistency to enable stable unlearning, thereby avoiding adverse effects on the retention capacity of the model.
\item Extensive experiments across in-domain and out-of-domain benchmarks demonstrate that LEGATO achieves state-of-the-art performance while fine-tuning 95\% fewer parameters and 67\% reduction in parameter update time for generative identity unlearning.  
\end{itemize}

\section{Related Work}
\textbf{Machine Unlearning in Generative Models.} 
Mutual information \cite{li2024machine} serves as a bridge to achieve forgetting in image-to-image generative models by minimizing the L2 loss between representations of the forget samples and Gaussian noise. SalUn \cite{fan2024salun} is a saliency-guided unlearning framework that enables efficient and effective machine unlearning in both image classification and text-to-image generation models by selectively updating salient model weights. The Restricted Gradient method \cite{ko2024boosting} removes conflicts between forgetting and retaining objectives by orthogonalizing their gradients, preserving only the components beneficial to each task.

DoCo \cite{wu2025unlearning} and Score Forgetting Distillation (SFD) \cite{chen2025score} achieve effective concept unlearning in diffusion models through adversarial training and distilled alignment, respectively, but both require fine-tuning the original model parameters. GUIDE \cite{seo2024generative} instead targets generative identity unlearning in GANs using a single image, steering the forgotten identity toward a target while preserving overall generation quality. In contrast, LEGATO formulates identity unlearning as a continuous, modular transformation, updating only lightweight Neural ODE adapters while keeping the generator frozen. This design avoids the instability and overhead of full-model fine-tuning, enabling controllable, efficient, and stable unlearning without degrading generative quality.

\textbf{Neural ODE and Applications.} Neural ODE, inspired by ResNets \cite{he2016deep}, model transformations as continuous flows, where each layer corresponds to a discretized ODE step. A numerical solver computes the forward pass, enabling adaptive depth and continuous representation. During backpropagation, Neural ODEs use the adjoint sensitivity method \cite{pontryagin2018mathematical} for efficient gradient computation with constant memory, offering benefits like invertibility and smooth transitions. Neural ODE have been applied to diverse tasks including vision-language models \cite{zhang2025cross,zhang2024node}, medical imaging \cite{cheng2024continuous,cheng2025implicit}, time-series forecasting \cite{rubanova2019latent}, PDE solving \cite{yin2022continuous}, and large language models \cite{zhang2024unveiling}. Despite their broad utility, NODEs have not been explored for machine unlearning. In this work, we address this gap by being the first to leverage Neural ODEs for identity forgetting in generative models.

\begin{figure*}[t]
  \centering
  \includegraphics[width=\linewidth]{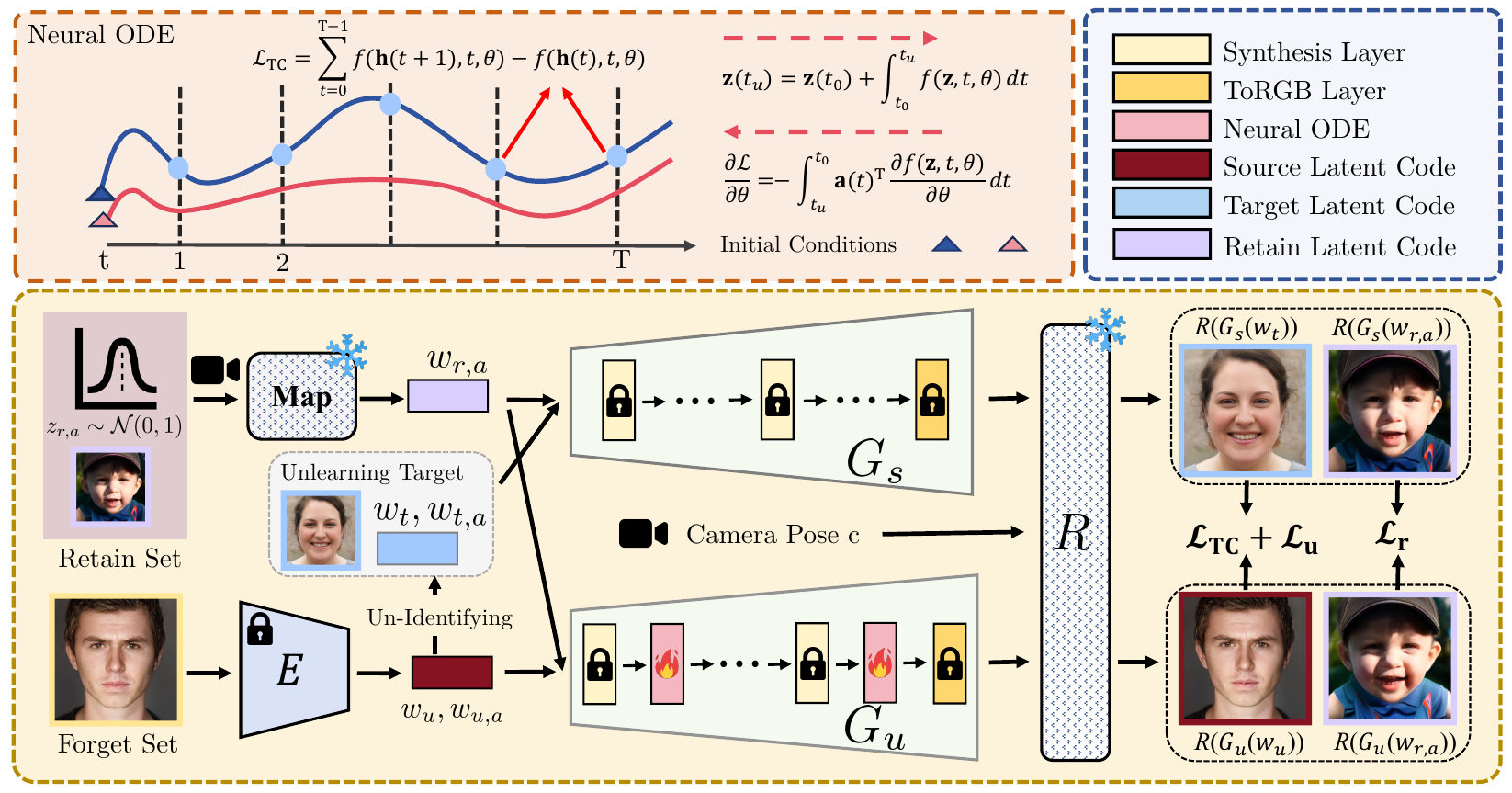}
  \caption{An overview of LEGATO. LEGATO introduces fine-tuned Neural ODE with fewer parameters, instead of fine-tuning the pretrained generator. Stable forgetting is achieved by imposing trajectory consistency constraint on the function. LEGATO aims to push the identity of the forget set toward a different one while preserving the generative ability for retained identities.}
  \label{fig:framework}
  \vspace{-10px}
\end{figure*}

\section{Method}
\subsection{Problem Formulation}
Given a GAN-based generative model EG3D \cite{chan2022efficient} and a single source image \(x_u \in \textbf{x}\) representing a specific identity, generative identity unlearning refers to the process of fine-tuning EG3D so that it is capable of reconstructing image \(\hat{x}_u \notin \textbf{x}\) from the latent code \(w_u\) of the source image, while maintaining generative ability for other identities. Specifically,
\begin{equation}
\label{Eq:Problem:Formulation}
\hat{x}_u = R(G_u(w_u);c), \; \text{where} \;  w_u= E(x_u), \;  \hat{x}_u \notin \mathbf{x}.
\end{equation}
In here, \(E\) denotes off-the-shelf inversion network \cite{yuan2023make} corresponding to EG3D, which encodes a given image into the latent code in the latent space of EG3D. \(G_u\) denotes the version fine-tuned from the pre-trained StyleGAN2 \cite{karras2019style} backbone \(G_s\) for the purpose of unlearning, and \(R\) is a fixed super-resolution module and \textit{c} denotes camera pose. After unlearning, \textit{multi-image test} is conducted by evaluating with a set of images $\{x_o^{i}\}_{i=1}^{N_o}$ from the same identity as $x_u$ ($x_u \neq x_o^{i}$), where $N_o$ denotes the number of such images.

\subsection{Method Overview}
In Figure \ref{fig:framework}, we provide an overview of our proposed LEGATO. The lower part illustrates the complete identity unlearning process. Given a source image \(x_u \in \textbf{x}\), we use an inversion network to obtain its latent code \(w_u\) and nearby codes \(w_{u,a}\) in the latent space. The unlearning targets \(w_t, w_{t,a} \notin \textbf{x}\) are then selected in reverse through the Un-Identifying strategy. The latent codes \(w_{r,a}\) of the retain set are sampled from a Gaussian distribution and mapped through the mapping network \(Map(\cdot)\) of EG3D. To preserve the generative capability for the retain set, we align the representations obtained by passing \(w_{r,a}\) through the pre-trained generator \(G_s\) and the fine-tuned generator \(G_u\), respectively. To achieve forgetting, we align the representations generated by passing \(w_t, w_{t, a}\) through the pre-trained generator \(G_s\) and \(w_u, w_{u,a}\) through the fine-tuned generator \(G_u\). \
\(G_u\) is built on Neural ODE, which act as an adapter layer that fine-tunes the generator’s parameters for identity unlearning. This architecture slashes the number of parameters that must be updated, yielding markedly greater training efficiency. In addition, the Neural ODE backbone learns a continuous transformation from the latent space to the image manifold, \textbf{allowing smoother and more stable forgetting to avoid catastrophic collapse}.

For effective generative identity unlearning, the objective of the unlearned model is to minimize the discrepancy between the unlearned image \(\hat{x}_u\), derived from \(w_u\),  and a target image \(\hat{x}_t\) from a different identity, derived from \(w_t\). Instead of selecting a random face or an average face generated by the mean latent code \(\overline{w}\) as the target image, we adopt the robust Un-Identifying strategy employed in GUIDE, which can be expressed as
\begin{equation}
\label{Eq:Un-Identifying}
w_t = \overline{w} - d \cdot \frac{w_{id}}{\Vert w_{id} \Vert_2}, \; \; w_{id} = w_u - \overline{w}, 
\end{equation}
where \(\overline{w}\) is the average calculated by \(Map(\cdot)\), and $d$ is a hyperparameter that controls the target image to deviate from the mean latent code.

To forget the identity of a given image \(x_u\), we need to consider the neighborhood of target and source latent codes embedded from \(x_u\) using \(E\). Specifically, with the scale sampled from a uniform distribution \(a^i \sim U (0, a_{max})\),  adjacency-aware latent code are defined as
\begin{equation}
\label{Eq:Adjacency:Latent}
\begin{aligned}
w_{u,a}^i = w_u + \Delta^i, \;\; w_{t,a}^i = w_t + \Delta^i, \\
\Delta^{i} \in \Delta = \{ \alpha^i \cdot \frac{w_{r,a}^i-w_u}{\Vert w_{r,a}^i-w_u \Vert_2} \}_{i=1}^{N_a},
\end{aligned}
\end{equation}
where \(a_{max}\) and \(N_a\) are hyperparameters. \(w_{r,a}\) is a latent code sampled from the random noise vector \(z_{r, a}\), i.e., \(w_{r,a} = Map(z_{r, a})\). Therefore, the optimization objective of our identity unlearning task can be formulated as:
\begin{equation}
\label{unlearn:loss:ori}
\min_{\theta} \underbrace{\mathcal{L}_{\text{u}}(\theta \mid w_u, w_t)}_{\text{Forget}} + \underbrace{\mathcal{L}_{\text{r}}(\theta \mid z_{r, a})}_{\text{Retain}} ,
\end{equation}
where \( \mathcal{L}_u \) denotes the loss for unlearning a specific identity, and \( \mathcal{L}_r \) represents the loss for preserving the generative capability on the retained set of identities. Details are provided in Section 4.1 of the supplementary material.

\subsection{Neural ODE Adapter for Unlearning}
\label{NODEs}

In this work, we introduce a parameter-efficient Neural ODE as an unlearning adapter, keeping the original model weights frozen to preserve generative capability and mitigate the adverse effects of excessive weight updates. An Neural ODE models the hidden state \textbf{$\boldsymbol{h}(t)$} as the solution of an initial‑value problem:
\begin{equation}
\boldsymbol{h}'(t) =  f(\boldsymbol{h}(t),  t, \theta), \quad  \boldsymbol{h}(t_0) = h_0 .
\end{equation}
In here, \(h_0\) represents the output of each synthesis layer, and $t \in \{0...T\}$. $\boldsymbol{h}(t)$ denotes the representation at each time step \(t\). $\theta$ are the parameters for the neural network.  Therefore, Neural ODE parameterized by  $\theta$ and governed by an ODE. In conventional feed‑forward networks, a very deep model demand substantially more memory. It will require a trade-off between accuracy and memory efficient.  In contrast, Neural ODE can be solved by an ODE solver in both forward and backward propagation which are more memory saving. In the forward pass, we view Neural ODE as an initial value ODE problem  and we can solve the solution by integration. We can express it in the following way: 
\begin{equation}
\mathbf{z}(t_u) = \mathbf{z}(t_0) + \int_{t_0}^{t_u} f(\boldsymbol{z},  t, \theta) dt.
\end{equation}
Then this integration form can be solved by and black-box ODE sovler
\begin{equation}
\mathbf{z}(t_u) = \text{ODESolve}(\mathbf{z}(t_0), f, \theta , t_0, t_u),
\end{equation}
where \(\text{ODESolve}(\cdot)\) refers to an ODE solver.
For the backward pass, we use another ODE solver and set $t_n$ as the staring point and $t_0$ as the final point. We can express the loss function in the following form:
\begin{equation}
\begin{aligned}
\mathcal{L}\bigl(\mathbf{z}(t_u)\bigr)
  &= \mathcal{L}\!\Bigl(
        \mathbf{z}(t_0)
        + \int_{t_0}^{t_u} f(\boldsymbol{z},  t, \theta)dt
     \Bigr) \\[6pt]
  &= \mathcal{L}\!\bigl(
        \text{ODESolve}\bigl(\mathbf{z}(t_0), f, \theta , t_0, t_u)
     \bigr).
\end{aligned}
\end{equation}
Then we can use the adjoint sensitivity method to compute the gradient and reduce the memory cost to O(1) memory cost.
We can compute it by:
\begin{equation}
\frac{\partial \mathcal{L}}{\partial \theta} = - \int_{t_u}^{t_0} \mathbf{a}(t)^T  \frac{\partial f(\boldsymbol{z},  t, \theta)}{\partial \theta} \, dt, 
\end{equation}
where $\mathbf{a}(t) = \frac{\partial \mathcal{L}}{\partial \mathbf{z}(t)}$ and $\frac{d\mathbf{a}(t)}{dt} = -\mathbf{a}(t) ^T \, \frac{\partial f(\boldsymbol{z},  t, \theta)}{\partial \mathbf{z}}$
We can solve all the  $\mathbf{z}, \mathbf{a},  \frac{\partial \mathcal{L}}{\partial \mathbf{z}(t)}$ with another ODE solver. 

\textbf{Neural ODE Flow.} Neural ODE adapt only the parameters that must change, thereby “unlearning” specific image features without perturbing the entire network. Each NODE defines a continuous vector field and solves an initial‑value problem, yielding unique trajectories in state space. This continuous‑time formulation leads to the following smoothness guarantee.
\begin{theorem}[Smooth Neural ODE Trajectories]
Let \( \Phi_{t_0 \rightarrow t} : [t_0, T) \times \mathbb{R}^d \times \Theta \rightarrow \mathbb{R}^d \), defined by \( \Phi_{t_0 \rightarrow t}(x_0, \theta) = \varphi(t; t_0, x_0, \theta) \), be the solution map of a Neural ODE parameterized by \( \theta \). If \( f \) is Lipschitz continuous in \( x \) and continuous in \( \theta \), then \( \Phi \) is of class \( \mathcal{C}^1 \). In particular, the solution is continuous and its Jacobians \( \partial_{x_0} \Phi \) and \( \partial_{\theta} \Phi \) exist and are continuous.
\end{theorem}
The complete theorem and proof is given in Section 1 of the supplementary material. Because Neural ODE learns a $\mathcal{C}^{1}$ flow, the model behaves more smoothly than discrete layers, enabling it to approximate target functions with higher retention capacity. This smooth theorem can benefit the unlearning process in two ways. Smoothness minimizes error per step and reduces accumulated error over time. In addition, a smooth path for the unlearning process can ensure that undesirable features are gradually removed rather than abruptly changed, thus mitigating catastrophic collapse during unlearning.

\begin{table*}[ht]
\centering
\setlength{\tabcolsep}{5pt}
\begin{tabular}{c|ccc|ccc|cccc}
\toprule
\multirow{2}{*}{Methods} & \multicolumn{3}{c|}{Random} & \multicolumn{3}{c|}{In-Domain (FFHQ)} & \multicolumn{4}{c}{Out-of-Domain (CelebAHQ)} \\
\cmidrule(lr){2-4} \cmidrule(lr){5-7} \cmidrule(lr){8-11}
& ID $\downarrow$ & \( \mathrm{FID}_{\text{pre}} \) $\downarrow$ & \( \Delta \mathrm{FID}_{\text{real}} \) $\downarrow$ & ID $\downarrow$ & \( \mathrm{FID}_{\text{pre}} \) $\downarrow$ & \( \Delta \mathrm{FID}_{\text{real}} \) $\downarrow$ & ID $\downarrow$ & \( \mathrm{ID}_{\text{avg}} \) $\downarrow$ & \( \mathrm{FID}_{\text{pre}} \) $\downarrow$ & \( \Delta \mathrm{FID}_{\text{real}} \) $\downarrow$ \\
\midrule
GUIDE & 0.10 & 10.29\text{\small$\pm$2.58} & 8.31\text{\small$\pm$1.58} & 0.06  & 7.77\text{\small$\pm$1.12} & 2.73\text{\small$\pm$0.84} & 0.02  & 0.23 & 7.44\text{\small$\pm$1.66} & 3.36\text{\small$\pm$1.12} \\
SalUn & 0.12 & 10.88 & 8.74 & 0.02 & 7.38 & 2.38 & -0.01 & 0.19 & 7.55 & 3.43 \\
RG    & 0.01 & 9.26  & 7.02 & 0.03  & 7.02 & 2.19 & -0.01 & 0.20 & 6.90 & 2.99 \\
DoCo  & -0.03 & 16.59 & 15.32 & 0.02 & 12.23 & 6.13 & -0.03 & 0.16 & 11.19 & 5.73 \\
LoRA  & 0.12 & 10.80 & 8.00 & -0.02 & 6.95 & 1.47 & -0.01 & 0.16 & 7.08 & 2.22 \\
\midrule
LEGATO & \textbf{-0.07} & \textbf{8.76}\text{\small$\pm$0.53} & \textbf{6.01}\text{\small$\pm$0.25} & 0.00 & \textbf{6.12}\text{\small$\pm$0.42} & \textbf{1.05}\text{\small$\pm$0.12} & 0.00 & 0.18 & \textbf{6.09}\text{\small$\pm$0.46} & \textbf{1.78}\text{\small$\pm$0.16} \\
Gains & - & \textbf{+15\%} & \textbf{+28\%} & - & \textbf{+21\%} & \textbf{+62\%} & - & \textbf{+22\%} & \textbf{+18\%} & \textbf{+47\%} \\ 
\bottomrule
\end{tabular}
\caption{Quantitative results of LEGATO and the baseline in the generative identity unlearning task, \( \mathrm{ID}_{\text{avg}} \) represents the results under the \textbf{multi-image} setting and the remaining results are under the single-image setting.}
\label{tab:overall：performance}
\vspace{-10px}
\end{table*}

\textbf{Controllability and Explainability.}
The rich theory of stability and error control for ordinary differential equations lets us put quantitative guarantees on the unlearning process. Regarding the controllability, choosing an explicit forward‑Euler solver for the ODE flow makes the process of unlearning procedure observable and auditable. 
For a fixed time‑step \(\Delta t\), the state update reads
\begin{equation}
  \boldsymbol{h}(t+1) = \boldsymbol{h}(t) + \Delta t \cdot f(\boldsymbol{h}(t),  t, \theta) .
  \label{eq:neural_ode_update}
\end{equation}

\begin{theorem}[Non-Monotonicity of Step Size]
\label{thm:nonmonotonic_stepsize}
Consider a Neural ODE unlearning process discretized by an explicit solver (e.g., Forward Euler) as Eq. \ref{eq:neural_ode_update}, and $\theta$ is updated via SGD. Let the total performance be $\mathcal{J}(\Delta t) = \mathcal{F}(\Delta t) + \mathcal{R}(\Delta t)$, with $\mathcal{F}$ quantifying forgetting and $\mathcal{R}$ retention. 
Then, under mild regularity conditions, $\mathcal{J}(\Delta t)$ is non-monotonic in $\Delta t$, and there exists a non-empty interval $[\Delta t_{\min}, \Delta t_{\max}]$ that optimally balances the forgetting–retention trade-off.
\end{theorem}

The discretization error in numerical methods is closely linked to the step size: large steps cause greater per-step and global errors, degrading retention ability, while overly small steps lead to instability and suboptimal performance due to mini-batch gradient noise, as illustrated in Theorem \ref{thm:nonmonotonic_stepsize}. Detailed theorem statement and proof are provided in Section 2 of the supplementary material. As a first-order method, the Euler method offers relatively good stability and supports a large range of step sizes for which convergence is guaranteed. In addition, a larger step size $\Delta t$ results in a greater magnitude of forgetting per step, leading to faster model updates and higher intensity of forgetting. In this sense, the step size offers an interpretable and controllable mechanism for regulating the forgetting strength.

\subsection{Trajectory Consistency Constraint}
Building on the previous subsection, a Neural ODE is a continuous, first‑order–differentiable dynamical system. To obtain smoother trajectories and to limit the negative impact that unlearning can have on the model’s generative ability of the retained data. We smooth the Neural ODE’s output during unlearning to achieve stable forgetting. This approach is referred to as Trajectory-Consistent Constraint, and the corresponding loss is given as follows:
\begin{equation}
\mathcal{L}_{\text{TC}} = \sum_{t=0}^{T-1} \| f(\boldsymbol{h}(t+1),  t+1, \theta) - f(\boldsymbol{h}(t),  t, \theta) \|_2^2.
\end{equation}
By smoothing the neural ODE, the learned vector field becomes locally more smooth, which improves the consistency of trajectory interpolation and extrapolation, and enhances robustness to small perturbations during training. In summary, our final objective is as follows:
\begin{equation}
\mathcal{L}_{\text{total}} =  \mathcal{L}_{\text{u}} + \mathcal{L}_{\text{TC}} + \mathcal{L}_{\text{r}}.
\end{equation}

\subsection{Conflict-Free Multi-Identity Unlearning}
In conventional discrete networks (e.g., LoRA or direct fine-tuning), unlearning multiple identities often leads to interference or conflicts \cite{ko2024boosting,Yike2024Resolving}. In contrast, the deterministic continuous flow defined by Neural ODEs with \textbf{non-intersecting trajectories} can effectively mitigate this issue. On one hand, different identities occupy distinct regions (or low-dimensional manifolds) in the latent space; the ODE flow thus continuously and cohesively transports an entire cluster of points corresponding to a specific identity toward a target region, without abruptly "jumping" onto the trajectory of another identity. On the other hand, the continuous flow induced by Neural ODEs closely resembles a homeomorphism, gradually pushing identities apart in the latent space rather than overwriting model parameters.

\begin{theorem}[Conflict-Free Multi-Identity Unlearning]
\label{thm:nonconflict}
Under Assumptions A1--A3, for any two distinct identities $i \neq j$ and any initial representations
$h_i(0) \in \mathcal{M}_i$, $h_j(0) \in \mathcal{M}_j$, the Neural ODE flow satisfies:

\begin{enumerate}
    \item \textbf{Trajectory Non-Intersection:}
    \[
    \Phi_t(h_i(0)) \neq \Phi_t(h_j(0)), \quad \forall t \in [0,T].
    \]

    \item \textbf{Manifold Non-Overlap:}
    \[
    \Phi_t(\mathcal{M}_i) \cap \Phi_t(\mathcal{M}_j) = \emptyset, \quad \forall t \in [0,T].
    \]

    \item \textbf{Forgetting--Retention Decoupling:}
    If $i \notin \mathcal{F}$, then
    \[
    \Phi_t(\mathcal{M}_i) \subset \mathcal{U}_i, \quad \forall t \in [0,T].
    \]
\end{enumerate}
\end{theorem}

\begin{remark}
LEGATO ensures unlearning trajectories of multiple identities do not intersect and remain non-overlapping at the manifold level, while the vector field within the regions corresponding to retained identities remains unaltered.
\end{remark}

\begin{figure*}[t]
  \centering
  \includegraphics[width=\linewidth]{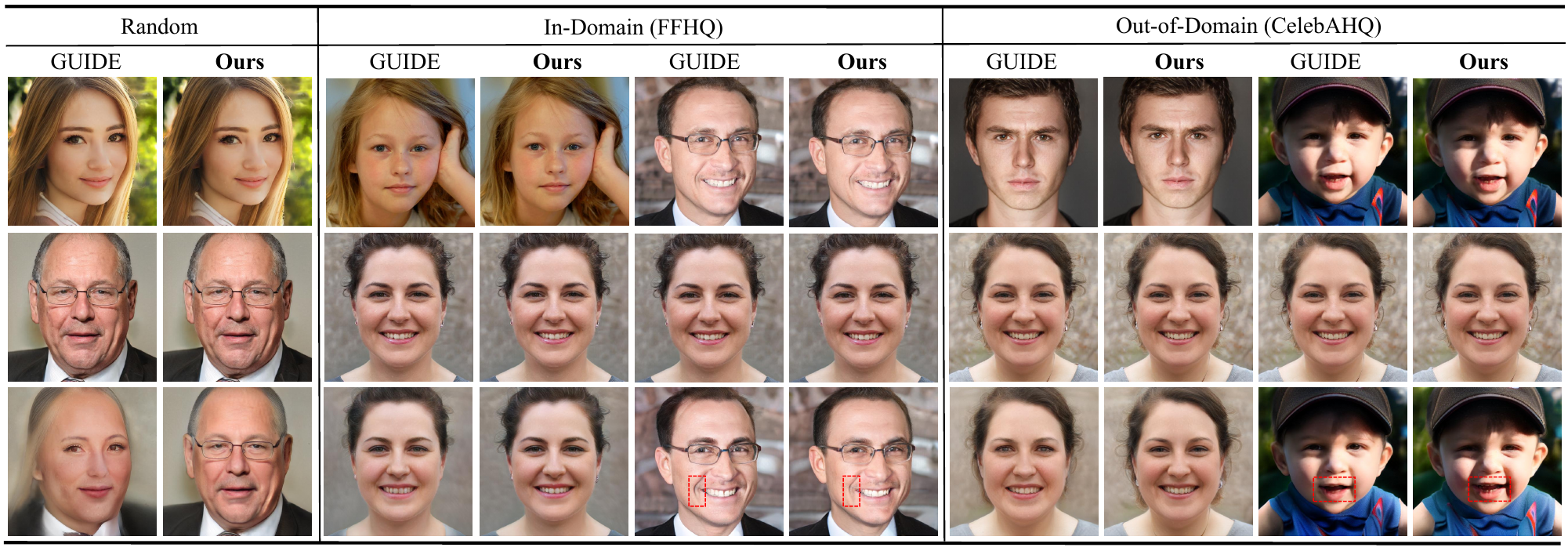}
  \caption{Qualitative results of GUIDE and the baseline in generative identity unlearning task. For the given source image each (the first
row), LEGATO aimed to erase the identity in the pre-trained generator while preserving the ability to generate other identities. The images in the second and third row are the target and unlearned images, respectively.}
   \label{fig:overall:result}
   \vspace{-10px}
\end{figure*}

\section{Experimental Results}
\subsection{Experimental Setting}

\textbf{Datasets.} We evaluate our method on three settings: (1) Random, where a source image is randomly sampled from the noise space; (2) InD (in-domain), where the source image is sampled from FFHQ \cite{karras2019style}, the pre-training dataset; and (3) OOD (out-of-domain), where the source image is sampled from CelebAHQ \cite{karras2018progressive}, which differs from the pre-training distribution. For the InD and OOD settings, latent codes are obtained using a GAN inversion network. Additionally, in the OOD setting, we perform multi-image evaluation on CelebAHQ by testing unlearning performance on other images sharing the same identity as the source image.

\textbf{Baselines.} We selected GUIDE \cite{seo2024generative}, the only available model for the generative identity unlearning task, as our baseline. Additionally, we implemented several methods from the concept unlearning task in generative models by ourself, such as DoCo \cite{wu2025unlearning}, RG \cite{ko2024boosting} and SalUn \cite{fan2024salun}. As a comparison to LEGATO, we also designed a LoRA-style \cite{Hu2022LoRA} approach by fine-tuning additional \textbf{discrete layers} to achieve unlearning, thereby comparing our method.

\textbf{Evaluation Metrics.} 
LEGATO’s performance was evaluated on unlearning (forget set) and retention (retain set). Unlearning was quantified via identity similarity (ID) from CurricularFace \cite{huang2020curricularface}, comparing images from identical latent codes before and after unlearning. A lower ID reflects greater dissimilarity—and thus stronger forgetting; we report the \( \mathrm{ID}_{\text{avg}} \) across a multi‑image test. This metric captures both global and local facial attributes. For retention capability, we evaluated distribution shifts by computing the Frechet Inception Distance (FID) score \cite{heusel2017gans} between the pre-trained and unlearned generators (\( \mathrm{FID}_{\text{pre}} \)), as well as the shift relative to real FFHQ images (\( \Delta \mathrm{FID}_{\text{real}} \)). A lower \( \mathrm{FID}_{\text{pre}} \) and \( \Delta \mathrm{FID}_{\text{real}} \) indicates better retention capability. Implementation Details of model can be found in Section 4.2 of the supplementary material.

\subsection{Overall Results}
\textbf{Numerical Results.} Table~\ref{tab:overall：performance} shows that LEGATO outperforms five unlearning baselines in both unlearning and retention capability. For unlearning, LEGATO achieves the best performance in \textit{Random}, and surpasses GUIDE while matching the strongest ID suppression methods in \textit{InD} and \textit{OOD}. Crucially, this privacy gain does not come at the cost of visual quality: LEGATO attains the lowest \( \mathrm{FID}_{\text{pre}} \) across all settings (8.76 in \textit{Random}, 6.12 in \textit{InD}, and 6.09 in \textit{OOD}) and the smallest degradation relative to real images (\( \Delta \mathrm{FID}_{\text{real}} \)), outperforming the next-best method by 14--29\%. Competing approaches exhibit a clear privacy--utility trade-off: methods with strong ID suppression (e.g., DoCo) nearly double FID, while those preserving moderate FID (e.g., RG, LoRA) perform similarly to GUIDE. Overall, LEGATO dominates the privacy--utility frontier, generalising from in-domain to out-of-domain data while effectively removing identity information and preserving high generative quality.

\begin{table}[t!]
\centering
\begin{tabular}{lcc}
\toprule
\textbf{Methods} & \textbf{Fine-tuning Prams} & \textbf{Time} \\
\midrule
GUIDE  & 28.20M &  4.9ms \\
DoCo   & 28.20M &  4.2ms \\
RG     & 28.20M &  3.6ms \\
LoRA   & 1.51M (-95\%)  & 1.6ms \\
LEGATO & \textbf{1.51M} (\textbf{-95\%})  & \textbf{1.6ms} (\textbf{-67\%}) \\
\bottomrule
\end{tabular}
\caption{Comparison of methods in terms of fine-tuning parameters and average parameter update time per epoch.}
\label{tab:methods:comparison:params}
\vspace{-10px}
\end{table}

Table \ref{tab:methods:comparison:params} highlights the sharp disparity in computational efficiency among the compared methods, as quantified by fine-tuning parameters and average parameter update time per epoch. Full-network approaches (GUIDE, DoCo, RG) must update around 28.20 million parameters, resulting in longer update times. In contrast, LoRA and LEGATO use lightweight adapters, updating just 1.51 million parameters—a remarkable 95\% reduction—leading to a 67\% reduction in update time (1.6ms per epoch). Importantly, LEGATO achieves superior retention performance compared to both LoRA and GUIDE, demonstrating that significant computational savings can be realized without compromising identity protection or overall effectiveness.

\textbf{Visual Results.} In Figure \ref{fig:overall:result}, we present the images generated by the unlearned model from the source image, and for FFHQ and CelebAHQ datasets, we also show the generation capability on the retain set. The results demonstrate that, under the same target settings, our method achieves better performance on the forgot set, while also better preserving the ability to generate fine details on the retain set (other identities), such as area in red box. More numerical and visual results can be found in the supplementary material.

\begin{table}[t]
\centering
\begin{tabular}{cccccc}
\toprule
Steps & Step size & ID & \( \mathrm{ID}_{\text{avg}} \) & \( \mathrm{FID}_{\text{pre}} \) & \( \Delta \mathrm{FID}_{\text{real}} \) \\
\midrule
4  & 0.10 & -0.01 & 0.18 & 6.85 & 2.29 \\
4  & 0.20 & -0.01 & 0.18 & 6.21 & 2.05 \\
4  & 0.40 & 0.00 & 0.18 & \textbf{6.09} & \textbf{1.78} \\
4  & 0.60 & -0.01 & 0.18 & 6.46 & 2.22 \\
4  & 1.00 & -0.01 & 0.15 & 7.59 & 3.22 \\
\bottomrule
\end{tabular}
\caption{Comparison of step size of Neural ODE and performance under multi-image setting (CelebAHQ).}
\label{tab:comparison:controllability}
\vspace{-10px}
\end{table}

\textbf{Controllability and Explainability.} As demonstrated in Table \ref{tab:comparison:controllability}, the choice of step size influences the retention capability in generative identity unlearning tasks. Conversely, an excessively small step size yields diminishing returns, providing only marginal gains while increasing computational overhead. This relationship mirrors the characteristics of classical numerical solvers, where the step size directly controls the numerical error.  Specifically, a moderate step size of approximately 0.4 achieves an optimal balance, delivering robust identity unlearning without triggering catastrophic collapse of generative capability. This clearly establishes controllability, enabling precise tuning of the forgetting intensity, and provides a transparent, explainable strategy for selecting effective operational parameters in Neural ODE-based unlearning frameworks.

\textbf{Robust to Noise Attack.} As shown in Table \ref{tab:comparison:noise:attack}, LEGATO exhibits significantly better robustness than GUIDE under noise attacks, where Gaussian noise is added to the test latent codes. We provide an intuition on why Neural ODE has better robustness. One of the well-known theorems in ODE is that the ODE solution trajectories never cross when the initial condition changes \cite{coddington1956theory}. In contrast, CNN does not have this property, and that's why Neural ODE has better robustness. 

\begin{table}[ht]
\centering
\begin{tabular}{ccccc}
\toprule
Methods & ID & \( \mathrm{ID}_{\text{avg}} \) & \( \mathrm{FID}_{\text{pre}} \) & \( \Delta \mathrm{FID}_{\text{real}} \) \\
\midrule
GUIDE     & 0.02 & 0.21 & 8.11 & 3.42 \\
LEGATO    & \textbf{0.00} & \textbf{0.17} & \textbf{6.98}  & \textbf{2.34}  \\
\bottomrule
\end{tabular}
\caption{Comparison of GUIDE and LEGATO under noise attack.}
\label{tab:comparison:noise:attack}
\vspace{-5px}
\end{table}

\textbf{Multi-Identity Unlearning.} The results in Table \ref{tab:comparison:multi:identity} show that LEGATO effectively mitigates conflict when unlearning multiple identities—evidenced by low $ID$ and $ID_{avg}$ scores—while preserving the generative capability on the retained set.

\begin{figure}[t]
  \centering
  \includegraphics[width=0.90\linewidth]{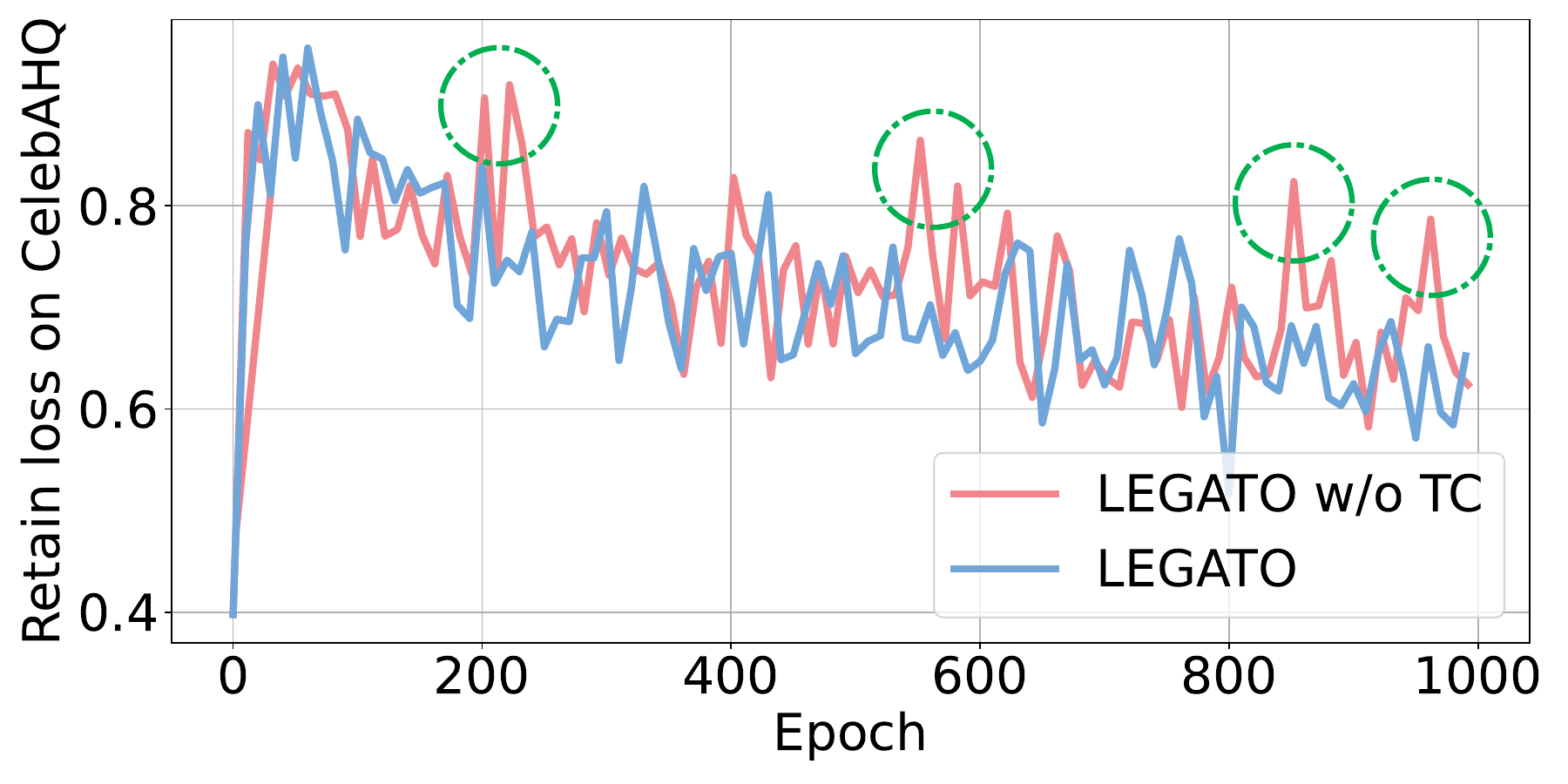}
  \caption{The Impact of TC on the retention loss.}
   \label{fig:ablation:retain:loss}
   \vspace{-10px}
\end{figure}

\begin{table}[ht]
\centering
\setlength{\tabcolsep}{3pt}
\begin{tabular}{cccccc}
\toprule
Multi-Identity & ID & \( \mathrm{ID}_{\text{avg}} \) & \( \mathrm{FID}_{\text{pre}} \) & \( \Delta \mathrm{FID}_{\text{real}} \) \\
\midrule
GUIDE-2nd  & 0.26 & 0.42 & 7.69 & 3.42 \\
LEGATO-2nd & -0.02(\text{\small+108\%}) & 0.19(\text{\small+55\%}) & 6.29 & 1.99 \\
\midrule
GUIDE-3rd  & 0.28 & 0.47 & 8.12 & 3.73 \\
LEGATO-3rd  & -0.02(\text{\small+107\%}) & 0.20(\text{\small+57\%}) & 6.34 & 1.87 \\
\bottomrule
\end{tabular}
\caption{Performance comparison of unlearning multiple identities (2 and 3 identities) on CelebAHQ dataset.}
\label{tab:comparison:multi:identity}
\vspace{-5px}
\end{table}

\subsection{Ablation and Sensitivity Studies}
\begin{table}[ht]
\centering
\begin{tabular}{cc|cccc}
\toprule
NODEs & TC & ID & \( \mathrm{ID}_{\text{avg}} \) & \( \mathrm{FID}_{\text{pre}} \) & \( \Delta \mathrm{FID}_{\text{real}} \) \\
\midrule
\ding{55}  & \ding{55} & 0.02  & 0.23 & 7.44 & 3.36 \\
\ding{51}  & \ding{55} & -0.02 & 0.16 & 6.88  & 2.20  \\
\ding{51}  & \ding{51} & 0.00  & 0.18 & \textbf{6.09}  & \textbf{1.78} \\
\bottomrule
\end{tabular}
\caption{Effectiveness of Neural ODE and Trajectory-consistent Constraint. TC represents Trajectory-consistent Constraint. We used CelebAHQ dataset in this experiment.}
\label{tab:ablation:study}
\vspace{-5px}
\end{table}

\begin{table}[t!]
\centering
\begin{tabular}{ccccccc}
\toprule
\(C_{\text{hidden}}\) & ID & \( \mathrm{ID}_{\text{avg}} \) & \( \mathrm{FID}_{\text{pre}} \) & \( \Delta \mathrm{FID}_{\text{real}} \) \\
\midrule
64   & 0.00  & 0.16  & 7.11 & 2.30 \\
128  & -0.02 & 0.25  & 6.81 & 2.41  \\
256  & 0.00  & 0.18  & \textbf{6.09} & \textbf{1.78}  \\
512  & -0.01 & 0.16  & 6.42 & 1.90  \\
\bottomrule
\end{tabular}
\caption{Comparison of neural function with different hidden layer dimensions under multi-image test (CelebAHQ).}
\label{tab:comparison:function:dimension}
\vspace{-5px}
\end{table}

\begin{table}[t!]
\centering
\begin{tabular}{ccccc}
\toprule
Solver & ID & \( \mathrm{ID}_{\text{avg}} \) & \( \mathrm{FID}_{\text{pre}} \) & \( \Delta \mathrm{FID}_{\text{real}} \) \\
\midrule
euler  & \textbf{0.00} & \textbf{0.18} & \textbf{6.09} & \textbf{1.78} \\
rk4    & 0.00 & 0.19 & 6.21 & 2.16  \\
midpoint  & 0.01 & 0.19 & 6.34  & 2.30  \\
\bottomrule
\end{tabular}
\caption{Comparison of different solver in Neural ODE under multi-image test (CelebAHQ).}
\label{tab:comparison:solver}
\vspace{-10px}
\end{table}

\textbf{Ablation Result.} 
In this section, we empirically analyze the individual contributions of (1) the Neural ODE module and (2) the Trajectory-consistent Constraint within our proposed framework. The ablation results are presented comprehensively in Table \ref{tab:ablation:study}. Our findings demonstrate that incorporating and fine-tuning the Neural ODE module substantially enhances the model's forgetting capability while significantly preserving the generative performance on the retain set. A comparison between Neural ODE and the discrete layers used in the LoRA-style approach further emphasizes the superiority of Neural ODE, highlighting its ability to mitigate negative impacts on retention capability. This improvement arises from Neural ODE's smooth and gradual forgetting mechanism, coupled with explicit controllability of forgetting intensity via step size adjustments. 

Moreover, the Trajectory-consistent Constraint (TC) plays a critical role in enhancing retention performance, particularly evident in the notable reduction of \( \mathrm{FID}_{\text{pre}} \). The impact of this constraint is vividly illustrated in Figure \ref{fig:ablation:retain:loss}, which depicts how trajectory consistency significantly improves stability during the final convergence phase (epochs 800 to 1000). Overall, these results underline the effectiveness of integrating Neural ODE and TC in achieving precise, controlled, and stable generative identity unlearning.

\textbf{Effect of hidden layer dimensionality and ODE solver.} As shown in Tables \ref{tab:comparison:function:dimension} and \ref{tab:comparison:solver}, different hidden layer dimensionalities and solver choices influence the generation capability on the retain set. Specifically, \(C_{\text{hidden}}\) = 256 yields the optimal performance among the tested dimensionalities, while the Euler solver consistently outperforms alternative solvers such as RK4 and midpoint. See supplementary material for additional ablation study results.

\section{Conclusion}
In this work, we introduce LEGATO, the first method leveraging Neural ODEs as fine-tunable adapters for generative identity unlearning, thereby avoiding the computational cost of full-model fine-tuning. Fine-tuning only the Neural ODE significantly reduces the impact on generative capability for retained data. In addition, LEGATO preserves generative quality on retained data through smooth, controllable forgetting, enhanced by our trajectory-consistent constraint that prevents catastrophic collapse. Extensive experiments confirm that LEGATO achieves state-of-the-art identity protection without compromising efficiency or performance.

\section{Acknowledgments}
This work is funded by National Natural Science Foundation of China (No. 62406347 and No. 62202302). Yi Chen was supported by the Hong Kong Research Grants Council, Early Career Scheme Fund [Grant 26508924] and Hong Kong RGC Theme-based Research Scheme T32-615-24/R. CWC and AIAR acknowledge support from the Swiss National Science Foundation (SNSF) under grant number 20HW-1 220785.AIAR gratefully acknowledges the support of the Yau Mathematical Sciences Center, Tsinghua University. This work is also supported by the Tsinghua University Dushi Program.

\appendix
\renewcommand{\thesection}{\arabic{section}}
\setcounter{equation}{0}
\setcounter{theorem}{0}
\setcounter{figure}{0}
\setcounter{table}{0}

\bibliographystyle{named}
\bibliography{ijcai26}
\urlstyle{same}

This Supplementary Material includes the complete proof of Theorem 1,Theorem 2 and Theorem 3, along with additional experimental details and results. 

\section{Proof of Theorem 1}
\begin{theorem}[Smooth Neural ODE Trajectories]
Let $f : [t_0, T] \times \mathbb{R}^d \times \Theta \longrightarrow \mathbb{R}^d, \quad (t, x, \theta) \mapsto f(t, x, \theta)$,
where \( \Theta \subseteq \mathbb{R}^p \) is an open parameter set.
Assume 
\\A1 (Local Lipschitz in \(x\)). For every compact \( K \subseteq \mathbb{R}^d \) and \( \theta \in \Theta \), there exists \( L_K = L(K, \theta) \) such that $\|f(t, x_1, \theta) - f(t, x_2, \theta)\| \leq L_K \|x_1 - x_2\| \quad \forall x_1, x_2 \in K, \; t \in [t_0, T]$. A2 Continuous in \((t, x, \theta)\). A3 (\( C^1 \) in \((x, \theta)\)). The partial derivatives \(\partial_x f\) and \(\partial_\theta f\) exist and are continuous on \([t_0, T] \times \mathbb{R}^d \times \Theta\).
Let \( \Phi_{t_0 \rightarrow t} : [t_0, T) \times \mathbb{R}^d \times \Theta \rightarrow \mathbb{R}^d \), defined by \( \Phi_{t_0 \rightarrow t}(x_0, \theta) = \varphi(t; t_0, x_0, \theta) \), be the solution map of a Neural ODE parameterized by \( \theta \). If \( f \) is Lipschitz continuous in \( x \) and continuous in \( \theta \), then \( \Phi \) is of class \( \mathcal{C}^1 \). In particular, the solution is continuous and its Jacobians \( \partial_{x_0} \Phi \) and \( \partial_{\theta} \Phi \) exist and are continuous.

\end{theorem}

\begin{proof}
    Because $f$ is continuous (assumption A2) and locally Lipschitz (Assumption A1) in $x$, the Picard–Lindelöf theorem yields, for every $(x_{0},\theta)\in\mathbb{R}^{d}\times\Theta$, a unique trajectory 
\begin{equation}
\varphi(\cdot;t_{0},x_{0},\theta)\in C^{1}\bigl([t_{0},T],\mathbb{R}^{d}\bigr)
\end{equation}
that solves the ODE.  
To establish continuity of $\Phi_{t_{0}\!\to t}$, fix $\theta$ and apply Grönwall’s inequality gives
\begin{equation}
\lVert\varphi(t;x_{0},\theta)-\varphi(t;x_{0}',\theta)\rVert
      \le e^{L(t-t_{0})}\lVert x_{0}-x_{0}'\rVert
\end{equation},
hence the flow depends Lipschitz-continuously on the initial state. Morover,
Because $f$ is continuous in $\theta$ and locally Lipschitz in $x$ uniformly in $\theta$,  the Continuous Parameter Dependence Theorem yields
\begin{equation}
  \bigl\lVert \varphi(t; x_0, \theta) - \varphi(t; x_0, \theta') \bigr\rVert
  \xrightarrow[\theta' \to \theta]{} 0 .
\end{equation}
is uniformly for  $t \in [t_0, T]$.
Therefore 
\begin{equation}
\Phi_{t_0 \to t} \in C^0(\mathbb{R}^d \times \Theta, \mathbb{R}^d).
\end{equation}

For differentiability, denote $J_{x}(t):=\partial_{x_{0}}\varphi(t)$ and differentiate the IVP with respect to $x_{0}$ to obtain the variational equation
\begin{equation}
\dot J_{x}(t)=\partial_{x}f\bigl(t,\varphi(t),\theta\bigr)J_{x}(t) \quad \textit{with} \quad J_{x}(t_{0})=I_{d} .
\end{equation} 
Because $\partial_{x} f$ is continuous, the same existence/uniqueness argument shows $J_{x}(t)$ exists and is continuous in $(x_{0},\theta)$. Hence $\Phi_{t_{0} \to t}$ is $C^{1}$ in $x_{0}$.

Similarly, writing $J_{\theta}(t):=\partial_{\theta}\varphi(t)$ and differentiating the IVP in $\theta$ gives
\begin{equation}
\dot J_{\theta}(t)=\partial_{x}f\bigl(t,\varphi(t),\theta\bigr)J_{\theta}(t)+\partial_{\theta}f\bigl(t,\varphi(t),\theta\bigr) .  
\end{equation}
$J_{\theta}(t_{0})=0_{d\times p}$.
This linear non‑homogeneous ODE again has a unique continuous solution, giving $\varphi(t)\in C^{1}$ in $\theta$.
Joint continuity of $J_{x},\, J_{\theta}$ follows from the coefficients’ continuity.
This completes the proof.
\end{proof}

\section{Proof of Theorem 2}

\begin{theorem}[Non-Monotonicity of Step Size] 
Consider a Neural Ordinary Differential Equation (Neural ODE) implemented via an explicit numerical solver (e.g., Forward Euler) for a continuous unlearning process. The hidden state evolves as
\begin{equation}
h_{k+1} = h_k + \Delta t \, f(h_k, \theta_k),
\end{equation}
where $f$ is a Lipschitz-continuous vector field and $\theta_k$ is updated using stochastic gradient descent (SGD).

Let the overall performance metric be defined as
\begin{equation}
\mathcal{J}(\Delta t) = \mathcal{F}(\Delta t) + \mathcal{R}(\Delta t),
\end{equation}
where $\mathcal{F}(\Delta t)$ measures the forgetting performance ($ID$)  and $\mathcal{R}(\Delta t)$ measures retention performance ($\Delta \mathrm{FID}_{\text{real}}, \mathrm{FID}_{\text{pre}}$).

Then, under mild regularity assumptions, $\mathcal{J}(\Delta t)$ is a non-monotonic function of the step size $\Delta t$, and there exists a non-empty interval
\begin{equation}
\Delta t \in [\Delta t_{\min}, \Delta t_{\max}],
\end{equation}
within which the forgetting--retention trade-off is optimal.
\end{theorem}

\paragraph{Assumptions.}
We make the following standard assumptions:

\begin{itemize}
    \item \textbf{A1 (Vector Field Regularity).} The function $f(h,\theta)$ is $L$-Lipschitz continuous with respect to $h$.
    \item \textbf{A2 (Stochastic Optimization Noise).} The parameter update follows
    \begin{equation}
    \theta_{k+1} = \theta_k - \eta \bigl(\nabla_\theta \mathcal{L} + \varepsilon_k\bigr),
    \end{equation}
    where $\mathbb{E}[\varepsilon_k] = 0$ and $\mathbb{E}\|\varepsilon_k\|^2 = \sigma^2$.
    \item \textbf{A3 (Finite Integration Horizon).} The total integration time $T$ is fixed, and the number of steps satisfies $N = T / \Delta t$.
\end{itemize}

\begin{proof}
We analyze the behavior of $\mathcal{F}(\Delta t)$ and $\mathcal{R}(\Delta t)$ in different step-size regimes.

\paragraph{Effect of small step size.}
The total state evolution over time $T$ can be written as
\begin{equation}
h(T) - h(0) = \sum_{k=0}^{N-1} \Delta t \, f(h_k, \theta_k).
\end{equation}
When $\Delta t$ is extremely small, the per-step deterministic update $\|\Delta t f(h_k)\|$ becomes negligible. Meanwhile, stochastic fluctuations induced by SGD and numerical discretization do not scale proportionally with $\Delta t$. As a result, the signal-to-noise ratio satisfies
\begin{equation}
\mathrm{SNR}(\Delta t) \propto \Delta t,
\end{equation}
which approaches zero as $\Delta t \to 0$. Consequently, the system enters a noise-dominated regime, leading to oscillatory local updates and degraded retention performance. Therefore,
\begin{equation}
\lim_{\Delta t \to 0} \mathcal{R}(\Delta t) \text{ increases}.
\end{equation}

\paragraph{Effect of large step size.}
When $\Delta t$ is large, the numerical integration error and discretization instability increase. Explicit solvers may violate stability conditions, causing the trajectory to deviate significantly from the smooth ODE flow and from the original generative manifold. This results in substantial degradation of retention capability, implying
\begin{equation}
\lim_{\Delta t \to \infty} \mathcal{R}(\Delta t) \to \infty.
\end{equation}

\paragraph{Non-monotonicity and optimal interval.}
The forgetting performance $\mathcal{F}(\Delta t)$ deteriorates for excessively small step sizes due to insufficient effective state evolution, while retention performance $\mathcal{R}(\Delta t)$ deteriorates for both excessively small and excessively large step sizes. Since $\mathcal{J}(\Delta t)$ is continuous with respect to $\Delta t$, by the Weierstrass extreme value theorem, there exists at least one minimizer
\begin{equation}
\Delta t^\star \in (\Delta t_{\min}, \Delta t_{\max}),
\end{equation}
corresponding to an optimal balance between forgetting and retention. When the number of steps $N$ is fixed, the per-step update magnitude varies with $\Delta t$. As a result, the source of non-monotonicity shifts from the accumulation of noise across steps to a mismatch in the dynamical system’s update scale, and the conclusion still holds. This completes the proof.
\end{proof}

\section{Proof of Theorem 3}

\begin{theorem}[Conflict-free Multi-Identity Unlearning]
\label{supp:thm:nonconflict}
Under Assumptions A1--A3, for any two distinct identities $i \neq j$ and any initial representations
$h_i(0) \in \mathcal{M}_i$, $h_j(0) \in \mathcal{M}_j$, the Neural ODE flow satisfies:

\begin{enumerate}
    \item \textbf{Trajectory Non-Intersection:}
    \[
    \Phi_t(h_i(0)) \neq \Phi_t(h_j(0)), \quad \forall t \in [0,T].
    \]

    \item \textbf{Manifold Non-Overlap:}
    \[
    \Phi_t(\mathcal{M}_i) \cap \Phi_t(\mathcal{M}_j) = \varnothing, \quad \forall t \in [0,T].
    \]

    \item \textbf{Forgetting--Retention Decoupling:}
    If $i \notin \mathcal{F}$, then
    \[
    \Phi_t(\mathcal{M}_i) \subset \mathcal{U}_i, \quad \forall t \in [0,T].
    \]
\end{enumerate}
\end{theorem}

\paragraph{Problem Setup.}
Let $\mathcal{H} \subset \mathbb{R}^d$ denote the latent (representation) space of a generative model.
Assume there exist $K$ distinct identities, each associated with a compact submanifold
\[
\mathcal{M}_k \subset \mathcal{H}, \quad k = 1,\dots,K,
\]
such that
\[
\mathcal{M}_i \cap \mathcal{M}_j = \varnothing, \quad \forall i \neq j.
\]

Each point $h \in \mathcal{M}_k$ is referred to as an \emph{identity representation}, meaning that
identity-related semantic information is encoded in the internal latent or feature representation $h$.

We model unlearning as a continuous-time dynamical system defined by a Neural Ordinary Differential Equation (Neural ODE):
\begin{equation}
\frac{d h(t)}{dt} = f(h(t), t; \theta), \quad h(0) = h_0,
\end{equation}
where $f: \mathcal{H} \times [0,T] \to \mathcal{H}$ is a neural vector field parameterized by $\theta$.

Let $\Phi_t : \mathcal{H} \to \mathcal{H}$ denote the solution (flow) map of the ODE such that
\[
\Phi_t(h_0) = h(t).
\]

\paragraph{Assumptions.} We make the following standard assumptions: 

\begin{itemize}
    \item \textbf{A1 (Lipschitz Continuity).}
    For each $t \in [0,T]$, the vector field $f(\cdot,t;\theta)$ is globally Lipschitz in $h$, i.e.,
    \[
    \| f(h_1,t) - f(h_2,t) \| \le L \| h_1 - h_2 \|, \quad \forall h_1,h_2 \in \mathcal{H}.
    \]

    \item \textbf{A2 (Identity Locality).}
    There exist disjoint open neighborhoods $\{\mathcal{U}_k\}_{k=1}^K$ such that
    \[
    \mathcal{M}_k \subset \mathcal{U}_k, \quad
    \mathcal{U}_i \cap \mathcal{U}_j = \varnothing \ \text{for } i \neq j.
    \]

    \item \textbf{A3 (Localized Unlearning).}
    The unlearning process modifies the vector field only inside forgotten identity regions:
    \[
    f(h,t;\theta) = f_0(h,t), \quad \forall h \notin \bigcup_{k \in \mathcal{F}} \mathcal{U}_k,
    \]
    where $\mathcal{F} \subset \{1,\dots,K\}$ denotes the set of identities to be forgotten.
\end{itemize}

\begin{figure*}[t!]
  \centering
  \includegraphics[width=\linewidth]{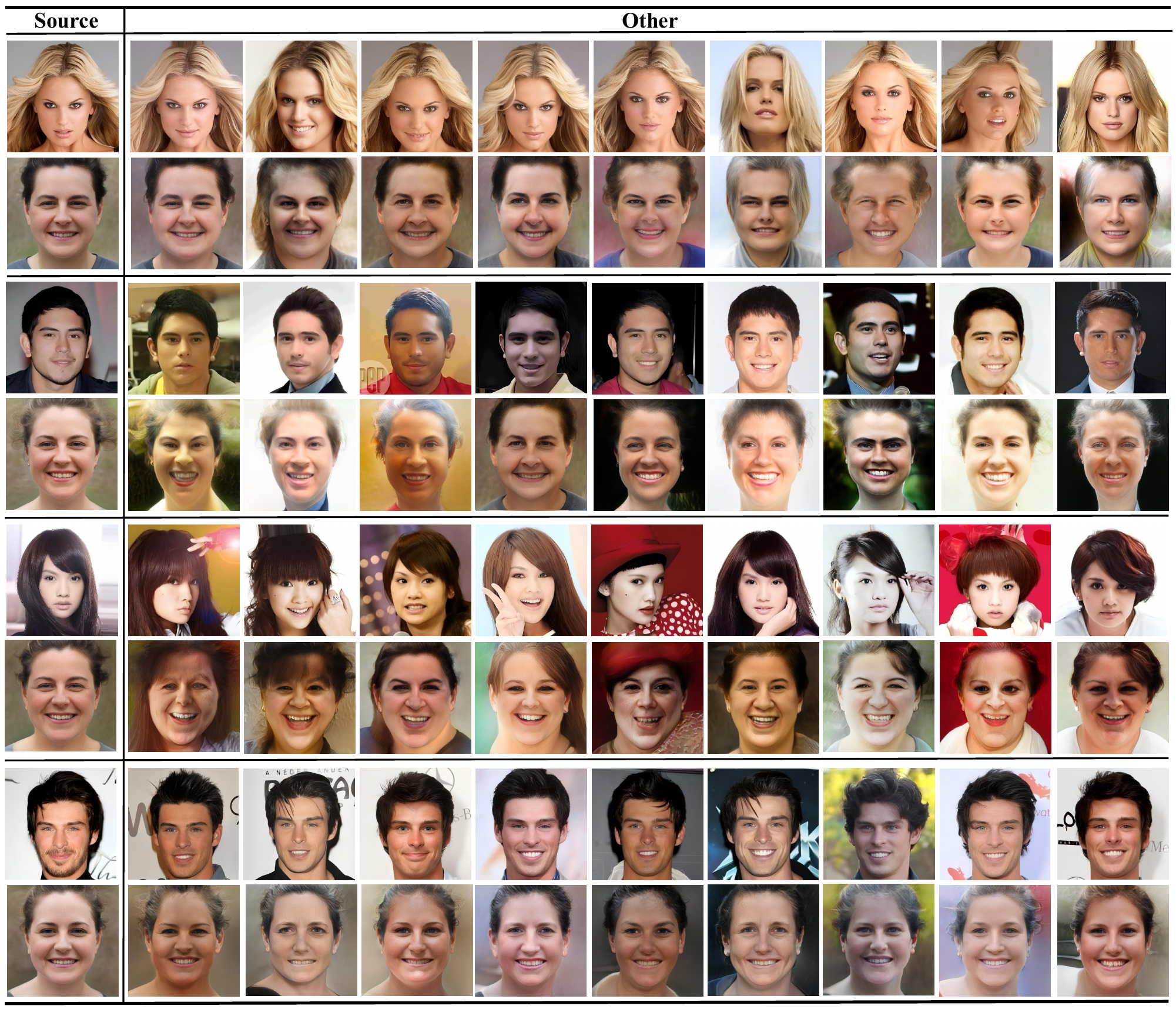}
  \caption{Qualitative results of LEGATO in generative identity unlearning task. For each identity in the CelebAHQ dataset, the first row shows the source image and other images of the same identity, and the second row displays the results after forgetting the specific identity. The identities are sequentially 1784, 3478, 7901 and 55.}
   \label{supp:fig:all:identity}
\end{figure*}

\begin{proof} We now proceed to analyze Theorem 3.
\paragraph{Existence and Uniqueness.}
By Assumption A1, the vector field $f$ is Lipschitz in $h$.
By the Picard--Lindelöf theorem, the ODE admits a unique solution for any initial condition on $[0,T]$.

\paragraph{Trajectory Non-Intersection.}
Assume for contradiction that there exists $t^\ast \in [0,T]$ such that
\[
\Phi_{t^\ast}(h_i(0)) = \Phi_{t^\ast}(h_j(0)).
\]
By uniqueness of solutions, this implies $h_i(0) = h_j(0)$, which contradicts
$\mathcal{M}_i \cap \mathcal{M}_j = \varnothing$.
Hence, trajectories cannot intersect.

\paragraph{Manifold Non-Overlap.}
The flow map $\Phi_t$ depends continuously on initial conditions.
Since $\Phi_t$ is injective and $\mathcal{M}_i, \mathcal{M}_j$ are disjoint compact sets,
their images under $\Phi_t$ remain disjoint for all $t \in [0,T]$.

\paragraph{Forgetting--Retention Decoupling.}
For any retained identity $i \notin \mathcal{F}$ and any $h \in \mathcal{U}_i$,
Assumption A3 implies the vector field coincides with the original one.
Thus, the trajectory remains within $\mathcal{U}_i$ by continuity and disjointness of neighborhoods.
\end{proof}

\section{Additional Implementation Details}
\subsection{Loss Function Design}
In this section, we present the concrete implementations of  $\mathcal{L}_u$  and  $\mathcal{L}_r$. Our forgetting loss consisting of Euclidean loss \(\mathcal{L}_2\), perceptual loss \(\mathcal{L}_{\text{per}}\) \cite{zhang2018unreasonable}, and identity loss \(\mathcal{L}_{\text{id}}\) \cite{deng2019arcface} is defined as:
\begin{equation}
\label{Eq:Unlearning:Loss}
\begin{aligned}
&\mathcal{L}_{\text{u}} = \mathcal{L}_{\text{local}} + \mathcal{L}_{\text{adj}}, \\
&\mathcal{L}_{\text{local}}(\hat{x}_u, \hat{x}_t) = \lambda_{\text{L2}} \mathcal{L}_2(F_u, F_t) + \lambda_{\text{per}} \mathcal{L}_{\text{per}}(\hat{x}_u, \hat{x}_t) \\
&\quad \quad \quad \quad \quad \quad  + \lambda_{\text{id}}\mathcal{L}_{\text{id}}(\hat{x}_u, \hat{x}_t) , \\
&\mathcal{L}_{\text{adj}}(w_u, w_t) = \frac{1}{N_a} \sum_{i=1}^{N_a} \mathcal{L}_{\text{local}}(\hat{x}_{u,a}^i, \hat{x}_{t,a}^i),
\end{aligned}
\end{equation}
where \(F_u = G_u(w_u)\) and \(F_t = G_s(w_t)\) are the tri-plane features of the backbone, $\hat{x}_{u} = R(F_{u})$ denotes the image reconstructed by the unlearned model from the source latent code, and $\hat{x}_{t} = R(F_{t})$ denotes the image reconstructed from the target latent code. $\hat{x}_{u,a}^{i}$ and $\hat{x}_{t,a}^{i}$ are the corresponding images reconstructed from their neighboring latent code.

To preserve the generative ability of other identities while forgetting a specific identity, we adopt the following form of retain loss:
\begin{equation}
\label{Eq:Unlearning:Retain:Loss}
\begin{aligned}
&\mathcal{L}_{\text{r}} = \frac{1}{N_r} \sum_{i=1}^{N_r} \mathcal{L}_{per}(\hat{x}_{u,r}^i, \hat{x}_{s,r}^i), \\
&\hat{x}_{u,r}^i=R(G_u(w_{r,a}^i);c), \hat{x}_{s,r}^i=R(G_s(w_{r,a}^i);c),
\end{aligned}
\end{equation}
where \(w_{r,a}^i\) is sampled from a random noise vector \(z_{r,a}\) and \(N_r\) denotes the number of samples, serving as the size of the retain set. $\hat{x}_{u,r}^i$ and $\hat{x}_{s,r}^i$ are obtained from the unlearned and pre-trained generator, respectively.

\subsection{Hyperparameter Settings}
In this section, we provide a detailed explanation of some hyperparameters in the model. The neural function of the Neural ODE consists of two 1×1 convolutional layers with \(C_{\text{hidden}}=256\). The step size and the number of steps used in the Neural ODE solver are set to 0.4 and 4, respectively. To ensure a fair comparison, the hyperparameters, including \(a_{max}, N_a, N_r, \lambda_{\text{L2}}, \lambda_{\text{per}}\) and \(\lambda_{\text{id}}\) are set to the same values as in GUIDE. Please refer to Table \ref{supp:tab:hyperparameters:details} for the specific values, where the "value" column shows the values used in LEGATO, and the "range" column presents the values used in the ablation studies.

\begin{table}[ht]
\centering
\setlength{\tabcolsep}{8pt}
\begin{tabular}{c|cc}
\toprule
Hyperparameter & Value & Range \\
\midrule
$ d $ & 30 & [-30, 0, 10, 30, 60] \\
$ \alpha_{max} $ & 15 & - \\
$ N_a $ & 2 & [1,2,4] \\
$ N_r $ & 2 & [1,2,4] \\
$ \lambda_{\text{id}} $ & 0.1 & [1e-2, 0.1, 1.0] \\
$ \lambda_{\text{per}} $ & 1.0 & [1e-2, 0.1, 1.0] \\
$ \lambda_{\text{L2}} $ & 1e-2 & [1e-2, 0.1, 1.0] \\
\bottomrule
\end{tabular}
\caption{The hyperparameter settings in LEGATO.}
\label{supp:tab:hyperparameters:details}
\end{table}

For the parameter initialization of the neural function in Neural ODE, we adopt an initialization method similar to LoRA. We initialize the first convolution using Kaiming uniform initialization and zero-initialize the final convolutional layer to ensure the module initially acts as an identity mapping, facilitating stable and non-disruptive fine-tuning. The Adam optimizer is used across all experiments, regardless of the learning rate.

\begin{table}[t!]
\centering
\setlength{\tabcolsep}{6pt}
\begin{tabular}{c|cccc}
\toprule
\multirow{2}{*}{Identity} & \multicolumn{4}{c}{Out-of-Domain (CelebAHQ)} \\
\cmidrule(lr){2-5}
& ID $\downarrow$ & \( \mathrm{ID}_{\text{avg}} \) $\downarrow$ & \( \mathrm{FID}_{\text{pre}} \) $\downarrow$ & \( \Delta \mathrm{FID}_{\text{real}} \) $\downarrow$ \\
\midrule
1784   & -0.06  & 0.13 & 6.14 & 2.35 \\
3478   & -0.04  & 0.19 & 6.24 & 1.08 \\
7901   & 0.00   & 0.20 & 6.23 & 1.92 \\
55     & -0.01  & 0.22 & 6.93 & 1.89 \\
\bottomrule
\end{tabular}
\caption{Quantitative results of LEGATO under different identity in the generative identity unlearning task.}
\label{supp:tab:ablation:LEGATO:identity}
\end{table}

\begin{table}[t!]
\centering
\begin{tabular}{ccccccc}
\toprule
\(N_a\) & ID & \( \mathrm{ID}_{\text{avg}} \) & \( \mathrm{FID}_{\text{pre}} \) & \( \Delta \mathrm{FID}_{\text{real}} \) \\
\midrule
1  & -0.01  & 0.17   & 6.02  & 1.79 \\
2  & 0.00   & 0.18   & 6.09  & 1.78 \\
4  & 0.00   & 0.16   & 7.76  & 2.34  \\
\bottomrule
\end{tabular}
\caption{Comparison of different \(N_a\) under multi-image test. We used CelebAHQ dataset in this study, and keep \(N_g = 2\).}
\label{tab:comparison:ablation:Na}
\end{table}

\begin{figure}[t!]
  \centering
  \includegraphics[width=0.8\linewidth]{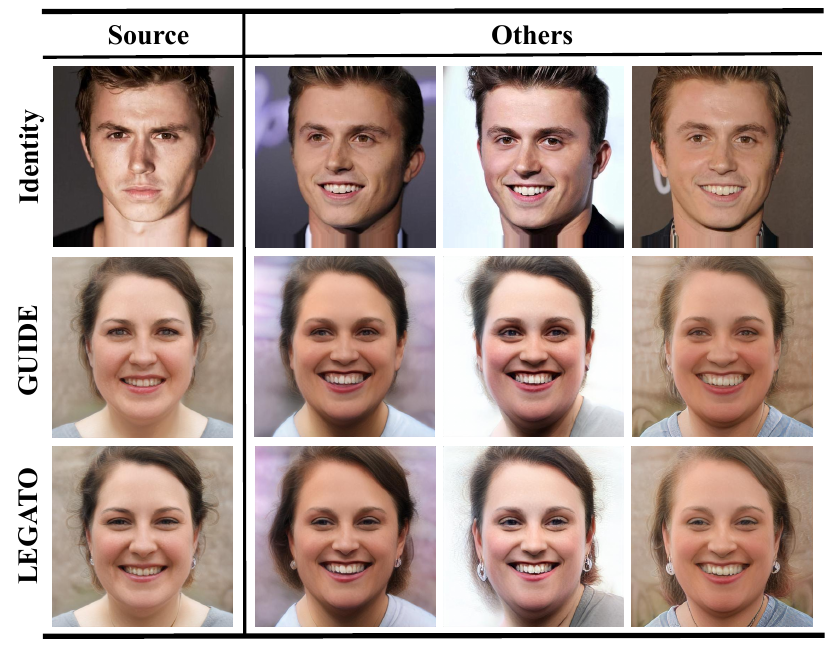}
  \caption{Qualitative results of LEGATO and the baseline on a multi-image test using CelebAHQ dataset. 
  }
   \label{fig:overall:result:forget:id1}
   \vspace{-10px}
\end{figure}

\begin{figure*}[t!]
  \centering
  \includegraphics[width=\linewidth]{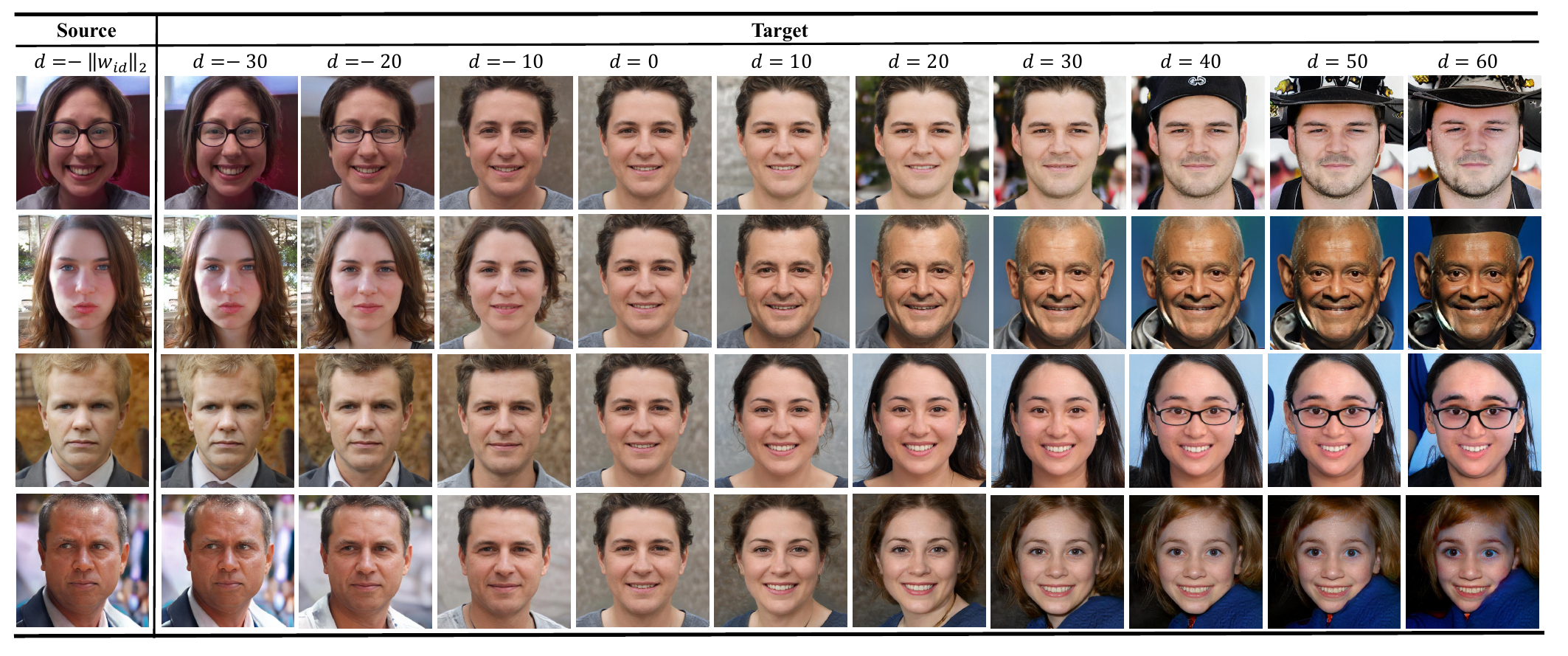}
  \caption{Illustration of target images from source images with different \(d\) in Random scenario.}
   \label{supp:fig:d:random}
\end{figure*}

\begin{figure*}[t!]
  \centering
  \includegraphics[width=\linewidth]{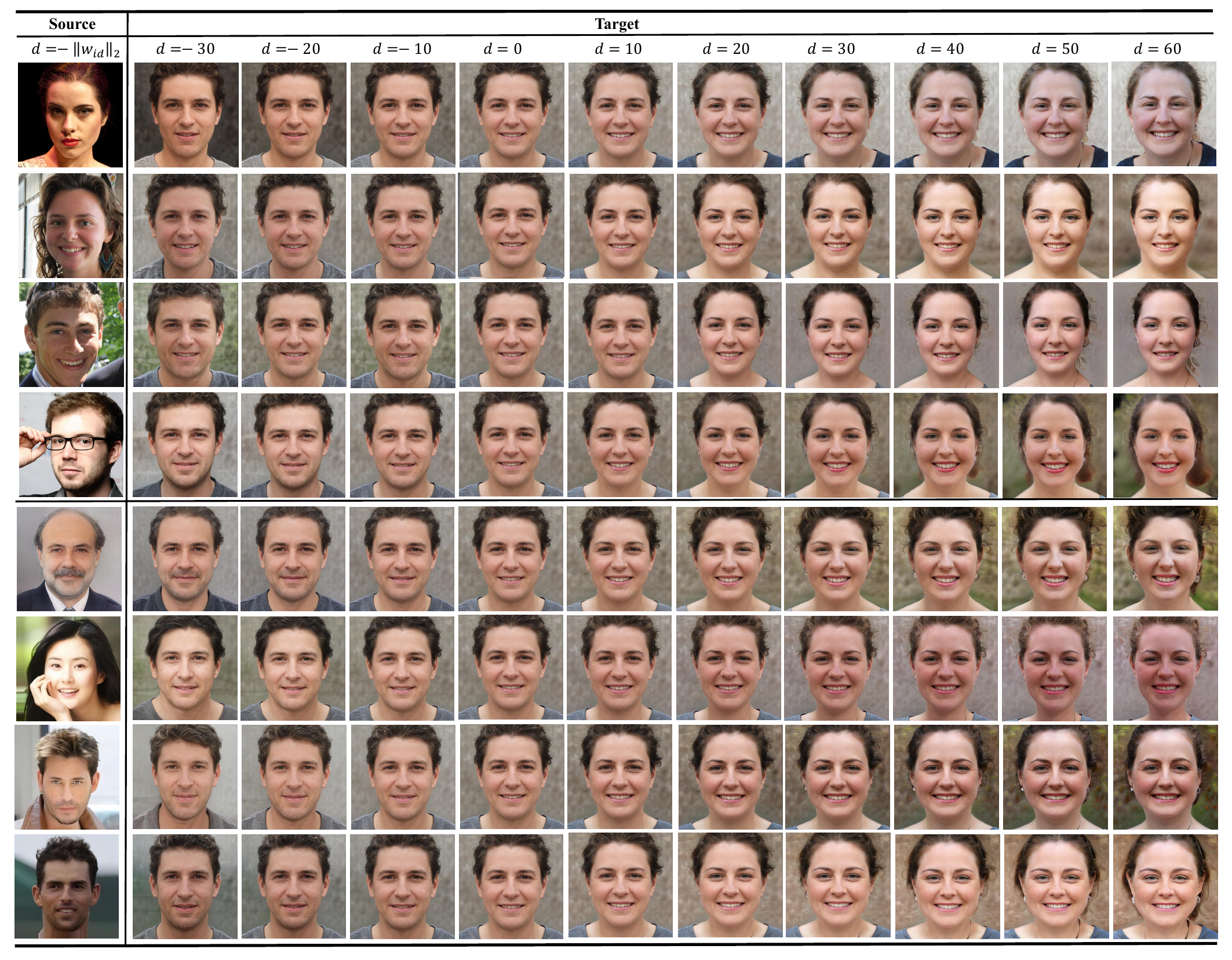}
  \caption{Illustration of target images from source images with different \(d\) in In-domain (FFHQ) and Out-of-domain (CelebAHQ) scenario.}
  \label{supp:fig:d:FFHQ}
  \vspace{10px}
\end{figure*}

\begin{figure*}[t!]
  \centering
  \includegraphics[width=\linewidth]{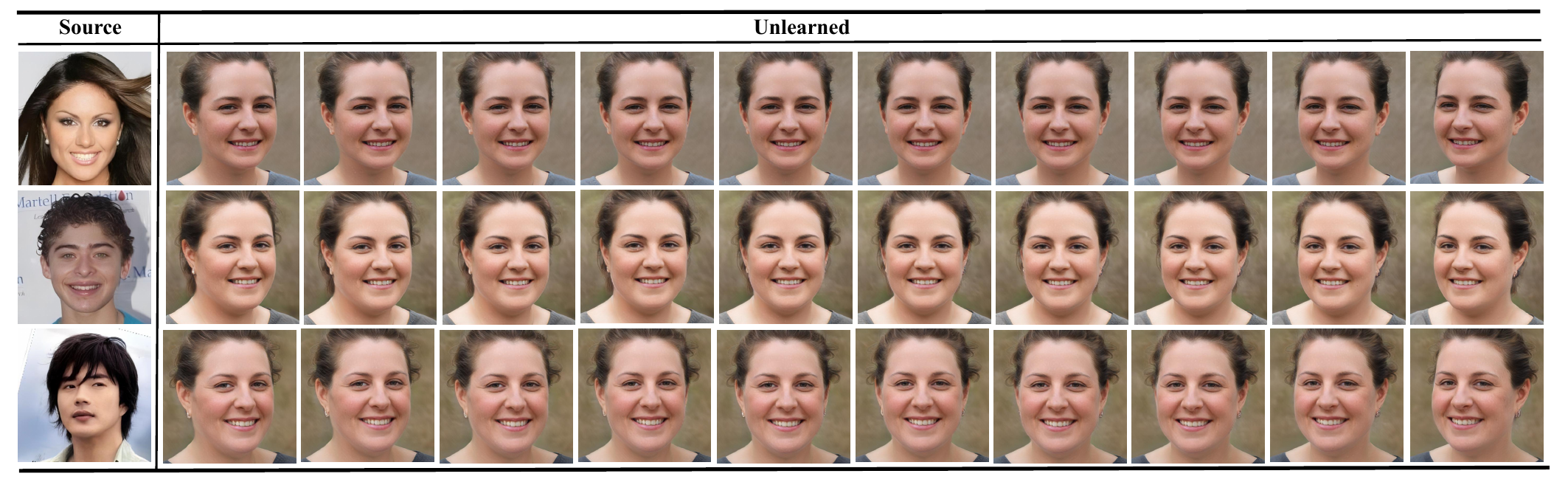}
  \caption{Unlearning results from different views in Out-of-domain (CelebAHQ) scenario.}
   \label{supp:fig:multi:view}
   \vspace{10px}
\end{figure*}

\begin{table}[t!]
\centering
\setlength{\tabcolsep}{6pt}
\begin{tabular}{c|cccc}
\toprule
\multirow{2}{*}{\textit{d}} & \multicolumn{4}{c}{Out-of-Domain (CelebAHQ)} \\
\cmidrule(lr){2-5}
& ID $\downarrow$ & \( \mathrm{ID}_{\text{avg}} \) $\downarrow$ & \( \mathrm{FID}_{\text{pre}} \) $\downarrow$ & \( \Delta \mathrm{FID}_{\text{real}} \) $\downarrow$ \\
\midrule
-30  & 0.22  & 0.55 & 4.13 & 1.39 \\
0    & 0.09  & 0.41 & 5.75 & 2.50 \\
10   & 0.04 & 0.36 & 6.44 & 2.86 \\
30   & 0.06 & 0.29 & 7.15 & 3.36 \\
60   & 0.05 & 0.30 & 8.94 & 3.62 \\
\bottomrule
\end{tabular}
\caption{Quantitative results of GUIDE under different \textit{d} in the generative identity unlearning task, identity id (celebAHQ) is 2161.}
\label{supp:tab:ablation:GUIDE:d}
\end{table}

For the Trajectory Consistency Constraint, we only apply it in the Neural ODE following the synthetic layer with a resolution of 128. We adopted a 128×128 rendering resolution for the triplane-based volumetric rendering module, followed by a super-resolution module that outputs final images at 512×512 resolution, consistent with the EG3D architecture built on StyleGAN2. Most of our experiments were conducted on an NVIDIA GeForce RTX 3090 GPU for 5 runs, while a small portion of experiments that exceeded the memory capacity were performed on an NVIDIA A100 GPU.

\section{Additional Experiments}

\subsection{Unlearning Results}

In this section, we present additional results of unlearning. Compared to our main paper, we used 10 images per identity in the CelebAHQ dataset, and the qualitative results are illustrated in Figure \ref{supp:fig:all:identity}. Table \ref{supp:tab:ablation:LEGATO:identity} sequentially presents the quantitative results of identity unlearning for these four identities. These results further quantitatively demonstrate that LEGATO effectively eliminates the specified identity not only in the provided source image but also across other images that share the same identity.

\begin{table}[t!]
\centering
\setlength{\tabcolsep}{6pt}
\begin{tabular}{c|cccc}
\toprule
\multirow{2}{*}{\textit{d}} & \multicolumn{4}{c}{Out-of-Domain (CelebAHQ)} \\
\cmidrule(lr){2-5}
& ID $\downarrow$ & \( \mathrm{ID}_{\text{avg}} \) $\downarrow$ & \( \mathrm{FID}_{\text{pre}} \) $\downarrow$ & \( \Delta \mathrm{FID}_{\text{real}} \) $\downarrow$ \\
\midrule
-30  & 0.18  & 0.59 & 5.09 & 1.40 \\
0    & -0.08 & 0.35 & 5.92 & 1.71 \\
10   & -0.08 & 0.31 & 6.38 & 2.05 \\
30   & 0.00 & 0.26 & 7.04 & 2.12 \\
60   & 0.09 & 0.21 & 8.25 & 2.58 \\
\bottomrule
\end{tabular}
\caption{Quantitative results of LEGATO under different \textit{d} in the generative identity unlearning task, identity id (celebAHQ) is 2161.}
\label{supp:tab:ablation:LEGATO:d}
\end{table}

\subsection{Target Images from Different \textit{d}} This section complements the main paper, “Effect of \(d\) in Determination of \(w_t\)”, by presenting additional experiments conducted on a wide range of source images. We visualized target images derived from a given source image
at multiple \(d\) values, as shown in Figures \ref{supp:fig:d:random} and \ref{supp:fig:d:FFHQ}. Our results illustrate that adjusting \(d\) allows us to get different target images. On the other hand, a smaller \(d\) leads to target images that are too similar to the source image, making unlearning difficult. In contrast, a larger \(d\) tends to distort the target images. Therefore, \(d=30\) is a reasonable choice.

Quantitative results in Tables \ref{supp:tab:ablation:GUIDE:d} and \ref{supp:tab:ablation:LEGATO:d} indicate: (1) Negative values of \(d\) can maintain the generative capability on the retain set but fail to achieve identity unlearning; (2) Excessively large \(d\) negatively impacts the generative performance on the retain set.

\subsection{Multi-View Unlearned Images}
In this section, we visualize unlearned images from continuous camera poses under the out-of-domain (CelebAHQ) scenario. As shown in Figure \ref{supp:fig:multi:view}, our unlearning process successfully erased the source identity in multiple camera poses.

\subsection{Visual Result on OOD dataset}
Figure \ref{fig:overall:result:forget:id1} presents visual results on the CelebAHQ dataset under a multi-image test, qualitatively demonstrating that LEGATO effectively achieves identity forgetting. 

\begin{table}[t!]
\centering
\begin{tabular}{ccccccc}
\toprule
\(N_g\) & ID & \( \mathrm{ID}_{\text{avg}} \) & \( \mathrm{FID}_{\text{pre}} \) & \( \Delta \mathrm{FID}_{\text{real}} \) \\
\midrule
1  & 0.01  & 0.15  & 9.40  & 3.70 \\
2  & 0.00  & 0.18  & 6.09  & 1.78 \\
4  & 0.00  & 0.21  & 5.42  & 1.29  \\
\bottomrule
\end{tabular}
\caption{Comparison of different \(N_g\) under multi-image test. We use CelebAHQ dataset in this study, and keep \(N_a = 2\).}
\label{tab:comparison:ablation:Ng}
\end{table}

\begin{table}[t!]
\centering
\begin{tabular}{ccccc}
\toprule
GUIDE & ID & \( \mathrm{ID}_{\text{avg}} \) & \( \mathrm{FID}_{\text{pre}} \) & \( \Delta \mathrm{FID}_{\text{real}} \) \\
\midrule
\(\lambda_1 = 1.0,\lambda_3 = 1.0\)   & 0.02  & 0.23 & 7.44 & 3.36 \\
\(\lambda_1 = 1.0,\lambda_3 = 0.5\)   & 0.18  & 0.34 & 8.58 & 3.99 \\
\(\lambda_1 = 1.0,\lambda_3 = 0.8\)   & 0.20  & 0.35 & 7.92 & 3.46 \\
\(\lambda_1 = 0.5,\lambda_3 = 1.0\)   & 0.22  & 0.37 & 6.71 & 2.56 \\
\(\lambda_1 = 0.8,\lambda_3 = 1.0\)   & 0.21  & 0.36 & 7.26 & 2.98 \\
\bottomrule
\end{tabular}
\caption{Comparison of different weights for final loss of GUIDE under multi-image setting (CelebAHQ). 
}
\label{tab:comparison:different:weights:guide}
\end{table}

\begin{table}[t!]
\centering
\begin{tabular}{ccccc}
\toprule
LEGATO & ID & \( \mathrm{ID}_{\text{avg}} \) & \( \mathrm{FID}_{\text{pre}} \) & \( \Delta \mathrm{FID}_{\text{real}} \) \\
\midrule
(1.0:1.0:1.0)   & 0.00  & 0.18 & 6.09 & 1.78 \\
(1.0:1.0:0.5)   & -0.01 & 0.14 & 9.33 & 3.47 \\
(0.5:1.0:1.0)   & \textbf{-0.02} & 0.17 & \textbf{5.78} & \textbf{1.60} \\
(1.0:0.5:1.0)   & 0.00  & 0.17 & 6.51 & 1.78 \\
\bottomrule
\end{tabular}
\caption{Effect of varying the loss-weight ratio \((\lambda_{1},\lambda_{2},\lambda_{3})\) in LEGATO on unlearning (ID, ID\textsubscript{avg}) and retention (FID\textsubscript{pre}, \(\Delta\)FID\textsubscript{real}) metrics under the multi-image CelebAHQ setting.  
The three numbers listed in the leftmost column are the relative weights \(\lambda_{1}:\lambda_{2}:\lambda_{3}\) used when computing the final loss.}
\label{tab:comparison:different:weights}
\end{table}

\section{Additional Ablation Study}
\subsection{Number of Latent Codes in Loss Functions}
In this section, we study the impact of \(N_a\) and \(N_r\), as shown in Table \ref{tab:comparison:ablation:Na} and \ref{tab:comparison:ablation:Ng}. The results indicate that a large \(N_a\) leads to worse generation performance on the retain set, while a larger \(N_g\) helps improve it. Moreover, \( N_a = N_g = 2 \) strikes a good balance between the unlearning and generation performance.

\subsection{Different Weights for Final Loss}

Under the same conditions as the main experiment, we investigated the impact of different weights on the final loss of GUIDE, denoted as $\mathcal{L}_{\text{GUIDE}} = \lambda_1 \mathcal{L}_{\text{u}} + \lambda_3 \mathcal{L}_{\text{r}}$. The experimental results in Table \ref{tab:comparison:different:weights:guide} demonstrate that adjusting the weights of the various terms in the GUIDE loss function does not achieve a better trade-off or improved interpretability. This is why we need a better unlearning model (LEGATO) to achieve a better trade-off and improved interpretability, avoiding negative impacts on the identity generation capability of the retained set while maintaining the forgetting ability.

Under the same conditions as the main experiment, we investigated the impact of different weights on the final loss of LEGATO, denoted as $\mathcal{L}_{\text{total}} = \lambda_1 \mathcal{L}_{\text{u}} + \lambda_2 \mathcal{L}_{\text{TC}} + \lambda_3 \mathcal{L}_{\text{r}}$. The experimental results in Table \ref{tab:comparison:different:weights} demonstrate that 1) The forgetting capability remains stable under different weight combinations; 2) By adjusting the ratios of the various terms in the loss function, the model's performance can even be further improved. However, these phenomena do not exist in the GUIDE model, fully demonstrating the effectiveness of our model.

\subsection{Scaling Factors of Loss Functions}
In this section, we study the impact of \(\lambda_{\text{L2}}\), \(\lambda_{\text{id}}\) and \(\lambda_{\text{per}}\), as shown in Table \ref{tab:comparison:ablation:lambda:L2}, \ref{tab:comparison:ablation:lambda:ID} and \ref{tab:comparison:ablation:lambda:per}. Experimental results show:(1) Excessively large \(\lambda_{\text{L2}}\) and \(\lambda_{\text{id}}\) lead to poor generative performance on the retain set; (2) An excessively small \(\lambda_{\text{per}}\) also results in degraded generation quality on the retain set. Therefore, a smaller \(\lambda_{\text{L2}}\) and \(\lambda_{\text{id}}\), along with a larger \(\lambda_{\text{per}}\), is a better trade-off between the unlearning performance and the retention performance. In conclusion, the final choice about \(\lambda_{\text{L2}}\), \(\lambda_{\text{id}}\) and \(\lambda_{\text{per}}\) represents a relatively optimal balance.

\begin{table}[t!]
\centering
\begin{tabular}{ccccccc}
\toprule
\(\lambda_{\text{L2}}\) & ID & \( \mathrm{ID}_{\text{avg}} \) & \( \mathrm{FID}_{\text{pre}} \) & \( \Delta \mathrm{FID}_{\text{real}} \) \\
\midrule
\(10^{-2}\)  & 0.00   & 0.18  & 6.09   & 1.78 \\
\(10^{-1}\)  & -0.02  & 0.14  & 9.54   & 3.98 \\
1            & -0.03  & 0.10  & 24.28  & 15.44  \\
\bottomrule
\end{tabular}
\caption{Comparison of different \(\lambda_{\text{L2}}\) under multi-image test. We use CelebAHQ dataset in this study.}
\label{tab:comparison:ablation:lambda:L2}
\end{table}

\begin{table}[t!]
\centering
\begin{tabular}{ccccccc}
\toprule
\(\lambda_{\text{id}}\) & ID & \( \mathrm{ID}_{\text{avg}} \) & \( \mathrm{FID}_{\text{pre}} \) & \( \Delta \mathrm{FID}_{\text{real}} \) \\
\midrule
\(10^{-2}\)  & -0.01  & 0.17  & 6.74  & 2.08 \\
\(10^{-1}\)  & 0.00  & 0.18  & 6.09  & 1.78 \\
1            & -0.01  & 0.15  & 7.09  & 2.49  \\
\bottomrule
\end{tabular}
\caption{Comparison of different \(\lambda_{\text{id}}\) under multi-image test. We use CelebAHQ dataset in this study.}
\label{tab:comparison:ablation:lambda:ID}
\end{table}

\begin{table}[t!]
\centering
\begin{tabular}{ccccccc}
\toprule
\(\lambda_{\text{per}}\) & ID & \( \mathrm{ID}_{\text{avg}} \) & \( \mathrm{FID}_{\text{pre}} \) & \( \Delta \mathrm{FID}_{\text{real}} \) \\
\midrule
\(10^{-2}\)  & -0.01  & 0.16  & 6.86  & 2.26 \\
\(10^{-1}\)  & -0.01  & 0.17  & 6.55  & 1.99 \\
1            & 0.00  & 0.18  & 6.09  & 1.78  \\
\bottomrule
\end{tabular}
\caption{Comparison of different \(\lambda_{\text{per}}\) under multi-image test. We use CelebAHQ dataset in this study.}
\label{tab:comparison:ablation:lambda:per}
\end{table}

\begin{table}[t!]
\centering
\begin{tabular}{cccccc}
\toprule
Steps & Step size & ID & \( \mathrm{ID}_{\text{avg}} \) & \( \mathrm{FID}_{\text{pre}} \) & \( \Delta \mathrm{FID}_{\text{real}} \) \\
\midrule
1  & 1.60 & -0.02 & 0.16 & 6.73 & 2.04 \\
2  & 0.80 & -0.02 & 0.16 & 6.57 & 2.16 \\
4  & 0.40 & 0.00  & 0.18 & \textbf{6.09} & \textbf{1.78} \\
8  & 0.20 & -0.01 & 0.18 & 6.66 & 2.20 \\
\bottomrule
\end{tabular}
\caption{Comparison of fixed integration intervals in Neural ODEs under multi-image setting (CelebAHQ).}
\label{tab:comparison:fixed:integration}
\end{table}

\subsection{Steps of Solver in Neural ODE}
In this section, we investigate the impact of different numbers of steps (or step sizes) on forgetting and retention performance under a fixed integration interval (i.e., 1.6) in ODEs, as shown in Table \ref{tab:comparison:fixed:integration}. The results show that even with a fixed integration interval, varying the step size or steps leads to different outcomes, and a step size of 0.4 achieves a favorable balance.

\section{Extend to More Architectures}
Although current identity unlearning approaches are primarily based on GAN architectures, we conducted a theoretical comparison with diffusion- and flow-matching-based architectures to evaluate the scalability of our method \cite{Shaheryar2025Unlearn}, and further performed experimental validation on the latest flow-matching-based architecture. Theoretically, current diffusion-based unlearning methods all involve fine-tuning the entire U-Net architecture, whereas our Node Adaptor can be easily inserted after each block—similar to how it is applied in GAN-based architecture. Moreover, full fine-tuning incurs computational complexity that grows with the scale of the U-Net, leading to prohibitively high computational costs. In the main text, we have adapted diffusion-based methods to the GAN architecture for comparison.

\begin{table*}[t!]
\centering
\begin{tabular}{ccccccc}
\toprule
Method & Retention (MMD $\downarrow$) & Retention (Accuracy $\uparrow$) & Forgetting (Forget Rate $\downarrow$) & Forgetting (Leakage $\downarrow$) \\
\midrule
Retrain     & 1.02e-3 & 97.6 & 0.5  & 5.3e-3 \\
Unlearn     & 0.32    & 61.6 & 0.2  & 2.7e-2  \\
Unlearn+KL  & 0.04    & 96.3 & 0    & 6e-3  \\
LORA        & 5.96e-3 & 96.6 & 0.1  & 1.7e-2  \\
LEGATO      & \textbf{0} & \textbf{98.3} & \textbf{0}  & \textbf{3.1e-3}  \\
\bottomrule
\end{tabular}
\caption{Experimental results on flow-matching-based architecture (MNIST dataset).}
\label{tab:comparison:flow:matching}
\vspace{-5px}
\end{table*}

Table \ref{tab:comparison:flow:matching} shows that LEGATO achieves strong forgetting performance and retention capability in the latest flow-matching architecture on MNIST dataset. The reason our results outperform the gold-standard retrain may be attributed to the introduction of new parameters in our adaptor. We have also provided the implementation code in our Git repository.

\end{document}